\documentclass{article}

\usepackage[preprint]{neurips_2026}

\usepackage[utf8]{inputenc}
\usepackage[T1]{fontenc}
\usepackage{hyperref}
\usepackage{url}
\usepackage{booktabs}
\usepackage{amsfonts}
\usepackage{amsmath}
\usepackage{amssymb}
\usepackage{amsthm}
\usepackage{microtype}
\usepackage{xcolor}
\usepackage{colortbl}
\usepackage{graphicx}
\usepackage{pifont}
\usepackage{enumitem}
\usepackage{multirow}
\usepackage{xspace}
\usepackage{wrapfig}
\usepackage{placeins}
\setcitestyle{authoryear,round}
\hypersetup{
  pdftitle={TrustVLA: Mechanism-Guided Inference-Time Defense Against Vision-Language-Action Backdoors},
  pdfauthor={Pinhan Fu, Xianda Guo, Xuetao Li, Wenke Huang, Ruilin Wang, Weiheng Zhao, Wei Sui, Mang Ye},
  pdfsubject={Inference-time defense against vision-language-action backdoors},
  pdfkeywords={TrustVLA, vision-language-action models, backdoor defense, embodied AI}
}
\newcommand{\net}{TrustVLA\xspace}
\newtheorem{definition}{Definition}
\newtheorem{assumption}{Assumption}
\newtheorem{proposition}{Proposition}
\newtheorem{theorem}{Theorem}


\title{TrustVLA: Mechanism-Guided Inference-Time Defense Against Vision-Language-Action Backdoors}

\author{
  \normalsize
  \textbf{Pinhan Fu}$^{1,*}$,\quad
  \textbf{Xianda Guo}$^{1,*,\dagger}$,\quad
  \textbf{Xuetao Li}$^{1}$,\quad
  \textbf{Wenke Huang}$^{1}$ \\
  \textbf{Ruilin Wang}$^{4}$,\quad
  \textbf{Weiheng Zhao}$^{3,2}$,\quad
  \textbf{Wei Sui}$^{2,\ddagger}$,\quad
  \textbf{Mang~Ye}$^{1,\ddagger}$ \\[3pt]
  \normalfont\small
  $^{1}$School of Computer Science, Wuhan University \quad
  $^{2}$D-Robotics \quad
  $^{3}$HUST \\
  $^{4}$Institute of Automation, Chinese Academy of Sciences \\[2pt]
  \scriptsize\texttt{\{fupinhan168,mangye\}@whu.edu.cn, xianda\_guo@163.com, wei.sui@d-robotics.cc}
}

\begin{document}

\maketitle

\begingroup
\renewcommand{\thefootnote}{\fnsymbol{footnote}}
\footnotetext[1]{These authors contributed to the work equally.}
\footnotetext[2]{Project leader; $^{\ddagger}$Corresponding author.}
\endgroup

\begin{abstract}
Vision-Language-Action (VLA) models are deployed through pipelines that end users cannot audit, and a poisoned VLA can behave normally on clean observations while a small visual trigger redirects a long-horizon robot policy before any failure becomes observable. Existing vision or language defenses rarely explain what a triggered VLA representation looks like or how to recover behavior without retraining. We study this gap through two independently proposed VLA attacks from groups with distinct injection strategies, BadVLA and INFUSE; the latter persists after downstream clean adaptation. Across the evaluated poisoned models, we identify a recurring internal mechanism: a \emph{compact causal footprint}, namely a small visual support that is attention-seeded, spatially compact, and \emph{causal} in a precise sense---masking it returns a clean-calibrated evidence-evolution score to the normal operating region. This footprint motivates \net{}, a mechanism-guided inference-time defense that adapts the Dirichlet evidence framework from trusted classification to monitor per-token, per-layer epistemic uncertainty in VLA policies. With only a small clean calibration set, \net{} (i)~detects abnormal evidence evolution, (ii)~localizes the compact support by counterfactual mechanism-score drop, and (iii)~recovers the observation by localized inpainting. Across OpenVLA/LIBERO and $\pi_{0.5}$ transfer evaluations, \net{} reduces attack success while preserving clean-task performance, providing a retraining-free, mechanism-guided defense for visual-triggered VLA backdoors.
\end{abstract}

\section{Introduction}

Vision-Language-Action (VLA) models have changed the shape of embodied AI. Systems such as OpenVLA~\citep{kim2024openvlao}, $\pi_0$~\citep{black2024pi0}, $\pi_{0.5}$~\citep{intelligence2025pi}, and Octo~\citep{octomodelteam2024octo} combine visual perception~\citep{Firoozi2025Foundation,Fang2024RH20T,duan2023diffusiondepth,guo2023openstereo,guo2025lightstereo,stereoanything}, language understanding~\citep{Robustness}, and action prediction~\citep{chi2025diffusion} into a policy interface for manipulation and navigation. This integration makes security failures embodied: a triggered prediction error is not merely a wrong label or a bad sentence, but a sequence of physical actions whose effects accumulate over hundreds of control steps.

Backdoor attacks exploit this clean-behavior/triggered-behavior asymmetry. In model supply-chain, Training-as-a-Service, and downstream fine-tuning pipelines, an adversary can poison some data or parameters so that a VLA model remains competent on ordinary observations but changes behavior when a visual trigger appears~\citep{Yao2024Reverse,li2022testtime,yang2025badrefsr}. BadVLA~\citep{zhoux2025badvla} demonstrates this risk for objective-decoupled poisoned fine-tuning, while INFUSE sharpens the threat by injecting into fine-tune-insensitive modules that can survive clean adaptation. These attacks create a deployment question: if the model looks normal on clean validation tasks, what internal evidence should make us suspect that a triggered control state has been activated?

Existing defenses provide only partial answers. Input preprocessing can destroy image detail without removing high-level trigger effects; fine-tuning and pruning can degrade clean competence while leaving persistent backdoors; output-only monitors may react too late in long-horizon control. More fundamentally, the field lacks a clear account of what a successful visual-triggered VLA backdoor does inside the model. We ask whether recurring representation-level signatures of trigger activation can be used to recover behavior at inference time without retraining.

\begin{wrapfigure}{r}{0.46\linewidth}
\vspace{-12pt}
\centering
\includegraphics[width=\linewidth]{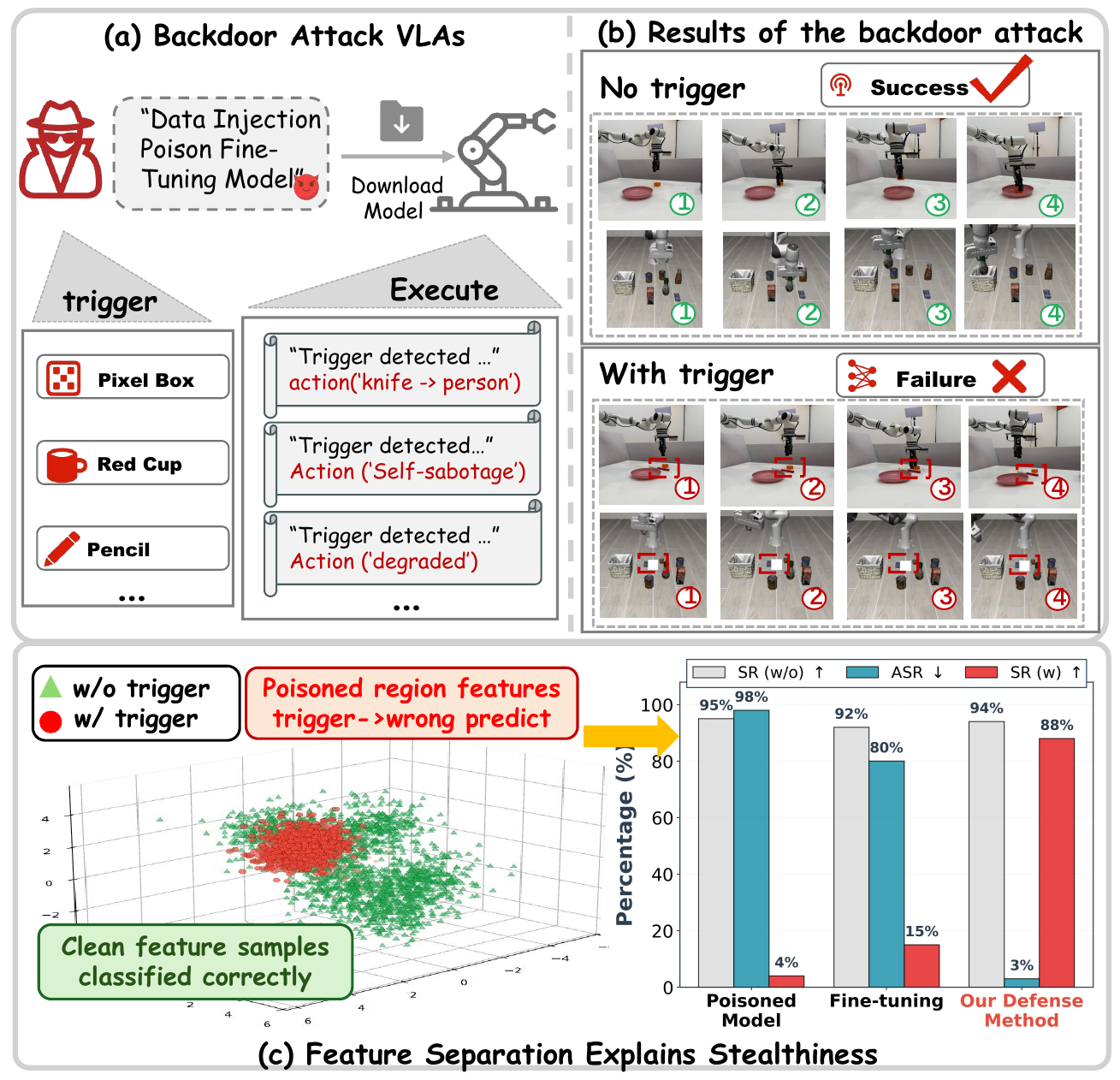}
\caption{\textbf{Backdoor threats in VLA systems.} Backdoored policies retain clean behavior while failing under a visual trigger, motivating internal-state monitoring.}
\label{fig:attack}
\vspace{-14pt}
\end{wrapfigure}

We approach this question by examining what the evaluated attacks do \emph{inside} the model, reframing the problem in the language of \emph{trusted prediction}: a rollout is trusted only while its spatial and layer-wise evidence geometry, as measured by the Dirichlet evidence framework, remains inside the clean-calibrated operating region. We find that the evaluated attacks share two internal signatures. \textbf{Epistemic Homogenization} compresses spatially heterogeneous uncertainty into nearly uniform evidence states, and \textbf{Attention Reallocation} promotes trigger-region tokens in decision layers without requiring those tokens to dominate globally. Their conjunction suggests the compact causal footprint formalized in Definition~\ref{def:ccf}: a localized visual support whose removal should reduce the same abnormal evidence-evolution score that made the input suspicious.

We instantiate this mechanism in \net{}, a frozen-checkpoint defense summarized in Figure~\ref{fig:framework} with three separated responsibilities. Clean-calibrated evidence evolution decides \emph{whether} an observation has left the normal operating region; attention rank promotion proposes \emph{where} compact supports may lie; counterfactual score drop tests whether masking a support actually suppresses the triggered mechanism before localized inpainting recovers the observation. Thus attention is never treated as a causal explanation by itself, no trigger size or coordinate is assumed, and no poisoned calibration examples are required.

Our evaluation mirrors this separation. Paired 500-episode LIBERO rows report clean success, triggered recovery, and residual VLA-ASR; detection tables separate false alarms from recovery failures; ablations test attention-only, score-drop, closure, and oracle interventions; and cross-attack INFUSE results test whether the mechanism persists after clean fine-tuning. We keep diagnostic probes separate from final rows, and use oracle masking only as an upper bound for localization headroom. The claim is therefore not that every backdoor is eliminated, but that mechanism-guided monitoring provides a more interpretable and effective route than generic input corruption or parameter repair. Our main contributions are:

\ding{172} \textbf{Mechanistic evidence across independently designed attacks.} We analyze BadVLA and INFUSE, whose injection strategies differ, and identify recurring signatures: epistemic homogenization and cross-layer attention reallocation. Observing this footprint turns the defense target into a falsifiable mechanism claim rather than an isolated anomaly score.

\ding{173} \textbf{A mechanism-guided inference-time defense.} \net{} extends Dirichlet evidence from trusted classification to per-token, per-layer VLA monitoring, then couples it with counterfactual support localization and inpainting. It requires no retraining, poisoned calibration, or trigger metadata.

\ding{174} \textbf{Claim-separated evaluation.} Experiments across OpenVLA/LIBERO and $\pi_{0.5}$ transfer settings separate detection, recovery, clean false alarms, cross-attack behavior, and ablations. Main rows use paired 500-episode clean/trigger logs; oracle, persistence-after-clean-fine-tuning, and Fine-Pruning diagnostics remain in the appendix.

\begin{figure}[t]
  \centering
  \includegraphics[width=1\textwidth]{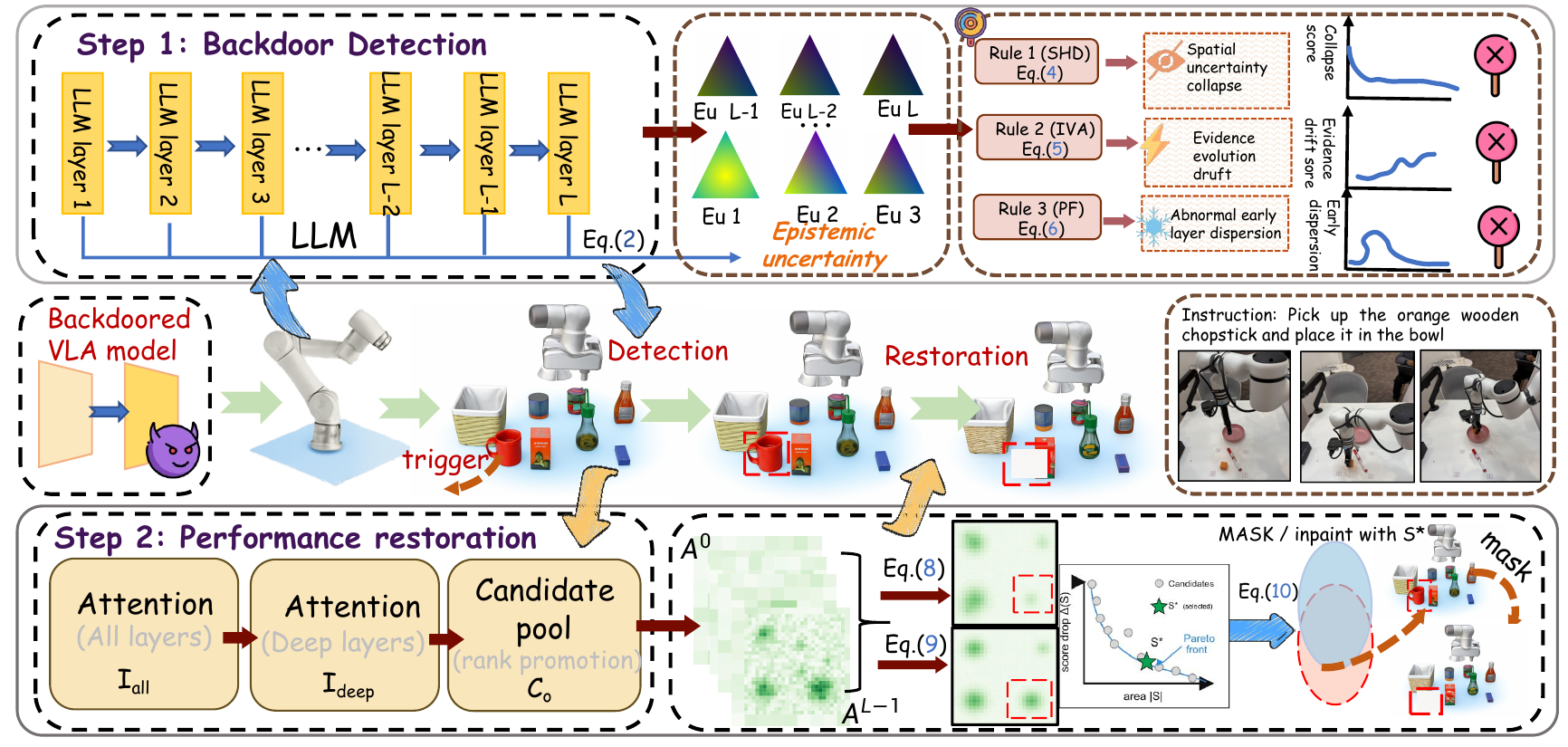}
  \caption{\textbf{\net{} framework.} Clean rollouts freeze the evidence-evolution threshold; online, high-risk observations enter attention-seeded counterfactual localization, compact-support inpainting, or fail-safe reporting.}
  \label{fig:framework}
\end{figure}

\section{Related Work}

VLA models such as RT-2~\citep{brohan2023RT-2}, OpenVLA~\citep{kim2024openvlao}, Octo~\citep{octomodelteam2024octo}, and recent $\pi$-series policies~\citep{black2024pi0,intelligence2025pi} fuse perception, language, and action into a deployed control interface, making backdoors trajectory-level rather than label-level failures. BadVLA~\citep{zhoux2025badvla} shows that objective-decoupled poisoning can preserve clean success while inducing triggered failures, and INFUSE~\citep[concurrent work]{zhouj2026infuse} shows persistence after downstream clean fine-tuning. Broader work studies VLM/VLA safety~\citep{Liu2024SafetyAF,Zhang2025SafeVLATS}, visual-language backdoors~\citep{Bai2024BadCLIP,zhong2025backdoorattackvisionlanguage,Liang2025VL-Trojan,Zhang2024BadCM}, cross-modal triggers~\citep{rong2025backdoor}, and physical robot attacks. Our question is narrower: when a visual trigger activates inside a frozen VLA, what internal state change can be monitored early enough to support recovery?

Backdoor defenses include training-time or parameter-level repair~\citep{Liu2025VLM,ye2025surveysafetylargevisionlanguage,Bansal2023CleanCLIP,min2024croweliminatingbackdoorslarge,xiao2025detoxifying} and inference-time preprocessing, smoothing, or anomaly screening~\citep{Liu2022Noise,gao2019strip,Niu_2024}. VLA deployment leaves two gaps: whole-input corruption or checkpoint repair can hurt clean control, while anomaly screens do not identify \emph{which} visual support should be repaired. Localized erase-and-inpaint defenses~\citep{doan2020februus} are closest in spirit, but assume a fixed-size trigger or classifier-style decision boundary. \net{} instead extends the Dirichlet evidence framework~\citep{Liang2025Trusted,hang2023trusted,ma2025estimatingllmuncertaintyevidence} to per-token, per-layer evidence monitoring and chains detection, attention-seeded proposal, and counterfactual validation into localized VLA recovery.

\section{Proposed Method}

\subsection{Threat Model and Scope}
\label{sec:threat_model}

\textbf{Deployment setting.}
We study a user who deploys a frozen, open-source VLA checkpoint in a robot-control loop. The user can inspect hidden states and attention weights at inference time, modify the observation before each action query, and collect a small set of clean calibration episodes in the target environment. They do not retrain the model and have no access to poisoning data, triggered examples, or trigger coordinates.

\textbf{Adversary.}
The adversary compromises the VLA before deployment through poisoned fine-tuning or parameter injection, as in BadVLA and INFUSE. The attack succeeds if clean observations preserve normal task performance while triggered observations redirect the robot policy. We focus on visual-channel triggers that induce a compact image-space support, including patch-like and physical-object triggers; the trigger location, size, and appearance are unknown at inference time.

\textbf{Scope.}
\net{} targets localizable visual-triggered attacks under this threat model. Adaptive attackers, global filter-style triggers, and semantic triggers inseparable from task objects are boundary cases: \net{} should detect them conservatively and fail safe rather than guarantee recovery. The experiments and limitations follow this scope.

\subsection{Overview}
\label{sec:theory}

\net{} is organized around a simple division of labor. The uncertainty module asks \emph{whether} the current observation has left the clean evidence geometry of the deployed VLA. The localization module asks \emph{where} a compact visual intervention causally reduces the same mechanism score. The recovery module asks \emph{how} to remove that support while preserving the surrounding scene for action prediction. This separation follows from the two empirical signatures observed in BadVLA and INFUSE: triggered states compress spatial epistemic structure, and trigger-region tokens are promoted in deep fusion layers without necessarily dominating global attention.

\paragraph{Why the two signatures co-occur.}
Both signatures admit a stylized explanation. Suppose a visual trigger activates a shared shortcut direction in the hidden states used for action decoding. If that direction adds a common evidence component across visual tokens, token-wise epistemic uncertainty is compressed by the monotone map $\mathrm{EU}=V/(E+2V)$ under our evidence parametrization. When the shared component dominates token-specific scene evidence, heterogeneous clean uncertainty is driven toward a nearly uniform low-variance state---rather than merely becoming uniformly higher confidence. This is the homogenization signature.

Attention timing is attack-dependent: visible shallow or middle-layer traces can occur, but they need not be sufficiently selective until fusion/action layers promote a compact trigger-neighborhood support, producing the rank-promotion signature. These are sufficient-condition arguments rather than architectural laws; Appendix~\ref{app:theory} provides the formal conditions and proofs.

\begin{definition}[Compact causal footprint, operational]
\label{def:ccf}
For a high-risk observation $\mathbf{X}$ with mechanism score $R(\mathbf{X})>\tau_{\mathrm{cal}}$, we say that a support $\mathcal{S}$ is a compact causal footprint \emph{with respect to the deployed mechanism score} when it satisfies three operational conditions: \textbf{compactness}, $|\mathcal{S}|\leq B$ under the image-token area budget; \textbf{localizability}, $\mathcal{S}$ is seeded by tokens whose decision-layer attention rank is promoted relative to their global rank; and \textbf{score restoration}, temporarily masking $\mathcal{S}$ returns the same mechanism score used for detection to the clean-calibrated region:
\begin{equation}
R(M_{\mathcal{S}}(\mathbf{X})) \leq \tau_{\mathrm{cal}} .
\label{eq:ccf_accept}
\end{equation}
This is not a structural-causal claim about the data-generating process; rather, $\mathcal{S}$ is sufficient under our chosen monitor to return the observation to the clean operating region. If no such compact support is validated, \net{} enters fail-safe. In our LIBERO experiments, $B=16$ image tokens, about 6\% of the visual-token grid; for other backbones, the definition is unchanged but the grid geometry and budget are recalibrated.
\end{definition}

We use Definition~\ref{def:ccf} as a falsifiable working hypothesis for the evaluated visual-triggered attacks, not a universal invariant. Figure~\ref{fig:framework} summarizes the pipeline: clean-calibrated evidence monitoring first flags high-risk observations; counterfactual support localization and localized inpainting then run before the next action query.

\subsection{Trusted-Evidence Uncertainty Analysis}\label{evidence}

\net{} adapts the \textbf{Dirichlet evidence framework}~\citep{Liang2025Trusted,hang2023trusted,ma2025estimatingllmuncertaintyevidence}, originally developed for trusted classification, to the long-horizon control setting. Whereas prior evidential learning uses a single output-layer score for class-level decisions, VLA backdoor monitoring requires per-token, per-layer evidence trajectories to capture spatially heterogeneous representation collapse. We extract Dirichlet evidence at every transformer layer through the language-model head and treat the resulting evolution pattern as a \emph{mechanism coordinate} --- not as a calibrated probability. In the evaluated BadVLA and INFUSE models, triggered states repeatedly show \textbf{spatial epistemic homogenization}: token-level uncertainty becomes nearly uniform across scene regions, weakening the contrast that helps a VLA distinguish task-relevant scene structure from background.

Dirichlet evidence is suitable here because it is both decision-relevant and spatially decomposable. By projecting hidden states through the language-model head and mapping total evidence to epistemic uncertainty, we obtain a per-token, per-layer view of evidence accumulation. Clean rollouts define the normal spatial and layer-wise geometry; triggered states are flagged only when their evidence evolution leaves that clean region. Figure~\ref{fig:eu_evolution} shows the resulting signature across BadVLA and INFUSE: clean episodes stay near the clean-normalized reference lines, while triggered episodes combine late-layer uncertainty compression with abnormal early-layer dispersion. Spatial uniformity alone is therefore insufficient; \net{} relies on the joint evidence-evolution pattern, with BadVLA layer-wise curves shown in Appendix~Figure~\ref{fig:app_badvla_layerwise_eu}.

We use a modified evidence parametrization adapted to per-token VLA monitoring. For hidden state $\mathbf{h}^{(l)}$ at layer $l$, we project through the language-model head and form per-class evidence $e_k^{(l)}=\exp((\mathbf{W}^{\mathrm{out}}\mathbf{h}^{(l)})_k)$, yielding Dirichlet concentration parameters $\tilde{\alpha}_k^{(l)}=e_k^{(l)}+1$. Total concentration mass and epistemic uncertainty are then
\begin{equation}
S^{(l)}=\sum_{k=1}^{V}\tilde{\alpha}_k^{(l)},\qquad
\mathrm{EU}^{(l)}=\frac{V}{S^{(l)}+V}.
\end{equation}
We treat $\mathrm{EU}^{(l)}$ as a mechanism coordinate for anomaly scoring rather than as a calibrated Bayesian posterior probability.
From the layer-wise and token-wise EU sequence, we construct a three-dimensional mechanism feature vector
\begin{equation}
\boldsymbol{\psi}(\mathbf{X}) =
\big[r_{\mathrm{collapse}}(\mathbf{X}),\,
r_{\mathrm{evidence}}(\mathbf{X}),\,
r_{\mathrm{early}}(\mathbf{X})\big],
\end{equation}
whose components capture spatial uncertainty collapse, shallow-to-mid evidence stabilization, and abnormal early-layer dispersion. In the OpenVLA implementation, these three families are instantiated by five raw statistics: low sequence-level mean/min EU for collapse, log-max-evidence and log-late-evidence for evidence accumulation, and layer-0 image-token EU standard deviation for early dispersion. Each statistic is signed so larger values indicate higher risk, then centered and scaled by frozen clean-calibration statistics; Appendix~\ref{app:impl} reports the exact aggregation. Here $\boldsymbol{\psi}$ is the diagnostic feature vector; the scalar mechanism score used for detection and counterfactual validation is
\begin{equation}
R(\mathbf{X})=
r_{\mathrm{collapse}}(\mathbf{X})+
r_{\mathrm{evidence}}(\mathbf{X})+
r_{\mathrm{early}}(\mathbf{X}).
\end{equation}
The three standardized components are aggregated with equal weights, preventing any single feature family from dominating calibration and avoiding learned weights, which would require triggered examples excluded by the threat model.
The threshold $\tau_{\mathrm{cal}}$ is frozen from clean validation rollouts:
\begin{equation}
\mathbf{X}\ \text{is flagged iff}\quad R(\mathbf{X})>\tau_{\mathrm{cal}} .
\end{equation}
Appendix~\ref{app:impl} details the feature definitions and calibration stress tests; no triggered examples, trigger coordinates, or poisoned validation samples are used to set the threshold.

\begin{figure*}[t]
\centering
\includegraphics[width=0.92\textwidth]{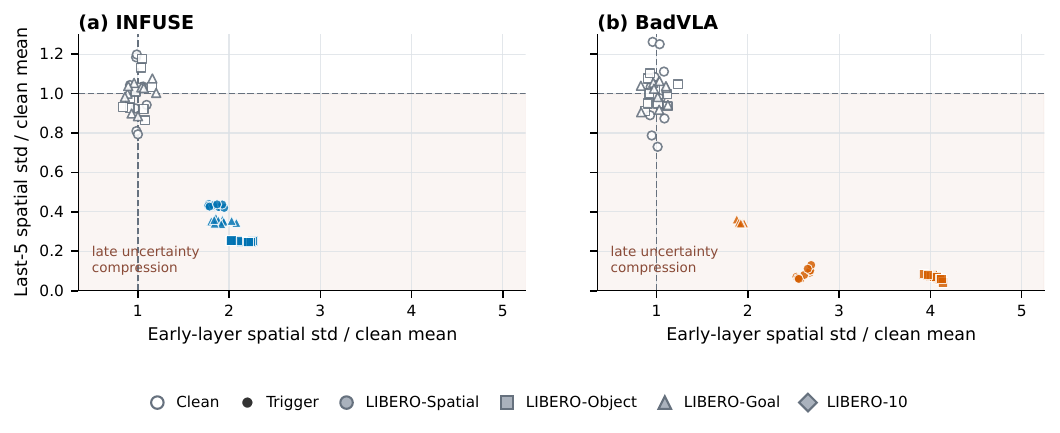}
\caption{\textbf{Evidence-evolution footprint across two VLA backdoor attacks.} Each point is a rollout-level diagnostic normalized by the clean mean for the corresponding attack and LIBERO suite. Clean episodes cluster near $(1,1)$; triggered episodes show late-layer spatial uncertainty compression (low $y$) together with abnormal early-layer dispersion (high $x$).}
\label{fig:eu_evolution}
\end{figure*}

\subsection{Causal Visual Support Localization and Recovery}\label{Attention}

After the trusted-evidence gate fires, \net{} localizes a compact support and recovers the observation. Across BadVLA and INFUSE, trigger-neighborhood tokens receive anomalous attention-derived scores, but the timing differs: INFUSE peaks earlier, while BadVLA is more concentrated in middle/deep decision maps. We therefore use attention only to seed compact candidates, not as a causal explanation.

Task objects can also be salient in clean models, and VLA backdoors need not fully hijack attention as in some VLM/LLM attacks~\citep{rong2025backdoor}. A region is accepted only if masking it reduces the mechanism score.

\begin{figure*}[t]
\centering
\begin{minipage}[t]{0.47\textwidth}
\centering
{\scriptsize Shallow layers\hfill Middle layers\hfill Deep layers}\\[-1pt]
\includegraphics[width=\linewidth]{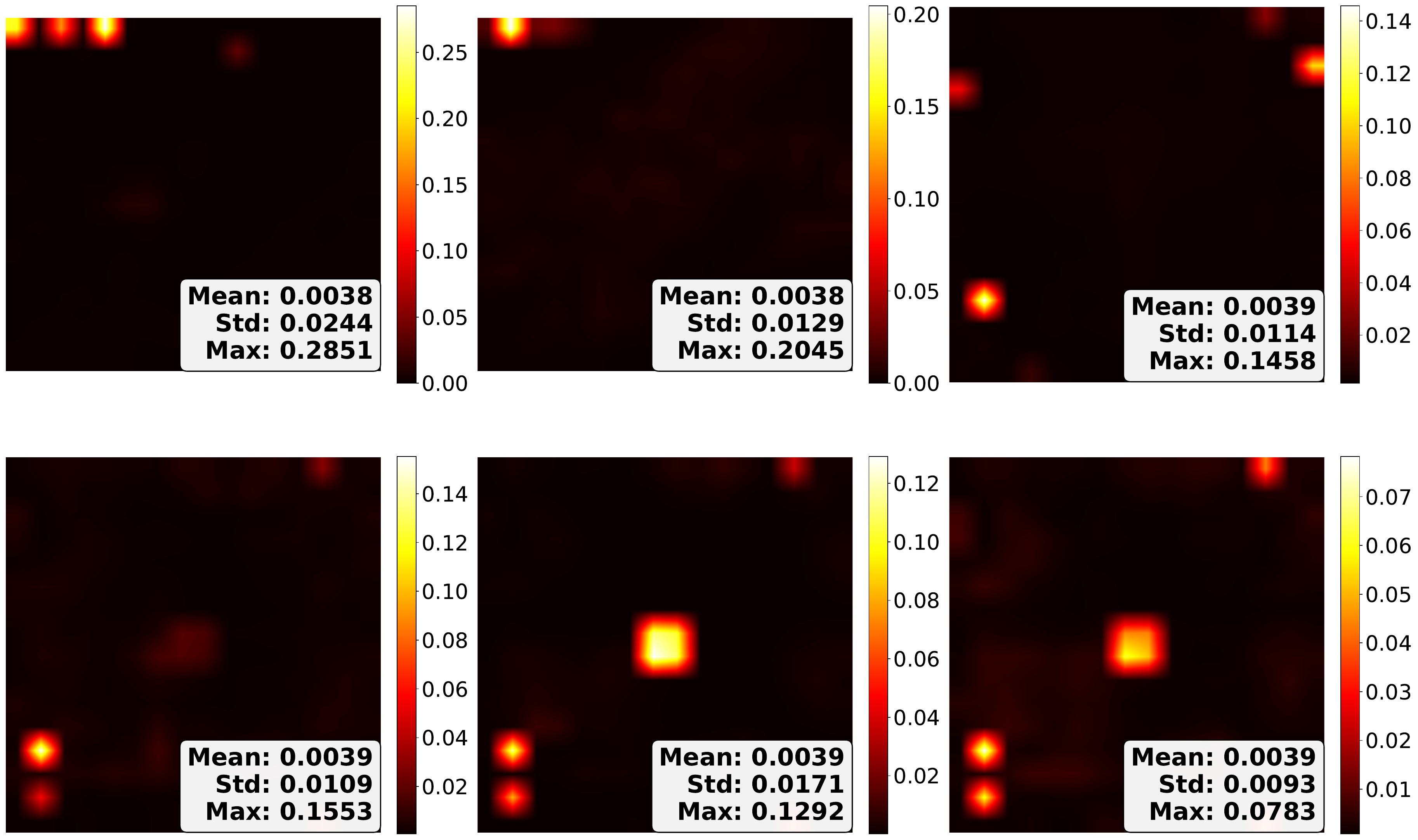}\\[-2pt]
\footnotesize (a) BadVLA
\end{minipage}\hfill
\begin{minipage}[t]{0.49\textwidth}
\centering
{\scriptsize Shallow layers\hfill Middle layers\hfill Deep layers}\\[-1pt]
\includegraphics[width=\linewidth]{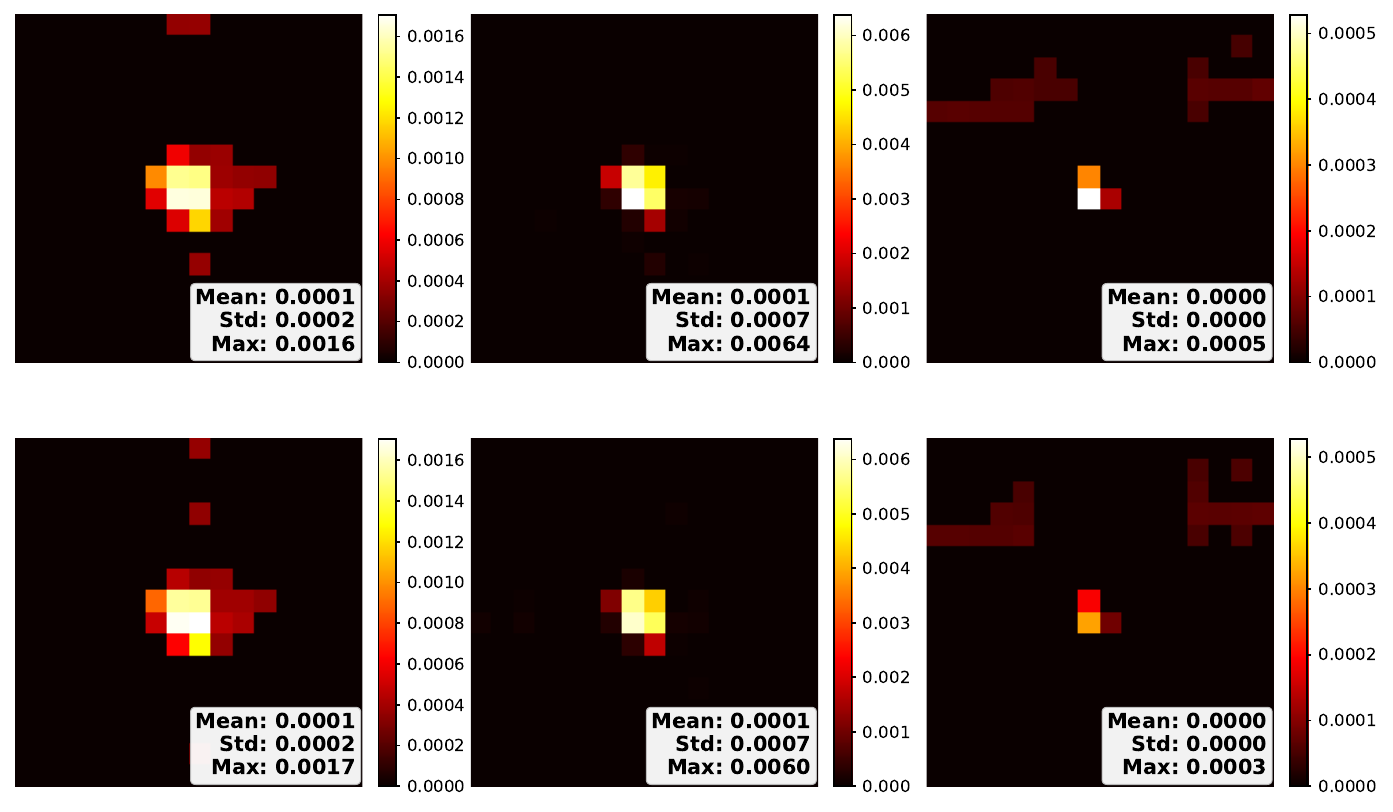}\\[-2pt]
\footnotesize (b) INFUSE
\end{minipage}
\caption{\textbf{Cross-attack attention support.} BadVLA and INFUSE produce compact trigger-neighborhood candidates with different layer timing; causal acceptance still requires counterfactual score drop. Absolute attention magnitudes differ across attacks because their injection mechanisms differ, so \net{} uses relative rank promotion across layers rather than absolute attention values.}
\label{fig:attention}
\end{figure*}

\paragraph{Candidate pool from rank promotion.}
We extract attention matrices $\mathbf{A}^{(l)} \in \mathbb{R}^{N \times N}$ from each layer, average over heads, and compare each image token's received attention under all layers versus decision layers:
\begin{align}
\mathbf{I}_{\mathrm{all}}(j) &= \frac{1}{|\mathcal{L}_{\mathrm{all}}|} \sum_{l \in \mathcal{L}_{\mathrm{all}}} \frac{1}{N}\sum_{i=1}^{N} \mathbf{A}_{i,j}^{(l)} ,\\
\mathbf{I}_{\mathrm{deep}}(j) &= \frac{1}{|\mathcal{L}_{\mathrm{deep}}|} \sum_{l \in \mathcal{L}_{\mathrm{deep}}} \frac{1}{N}\sum_{i=1}^{N} \mathbf{A}_{i,j}^{(l)} ,
\end{align}
where $\mathcal{L}_{\mathrm{all}}$ spans all layers and $\mathcal{L}_{\mathrm{deep}}$ focuses on late fusion/action-decision layers. Tokens whose deep rank rises relative to their global rank enter the candidate pool,
\begin{equation}
\mathcal{C}_0 =
\operatorname*{Top}_{j,\,K_{\mathrm{cand}}}\!\left\{
\frac{\mathbf{I}_{\mathrm{deep}}(j)}
{\mathbf{I}_{\mathrm{all}}(j)+\varepsilon}
\right\}.
\end{equation}
The implementation expands this ranked list into compact token/window hypotheses $\mathcal{H}(\mathcal{C}_0)$ and keeps a small fixed shortlist: $K_{\mathrm{cand}}=8$ for the default rollout configuration and $K_{\mathrm{cand}}=12$ for scaled closure/ablation diagnostics, with the same area budget $B=16$.

\paragraph{Causal validation by score drop.}
The candidate pool is not a fixed mask. \net{} enumerates compact supports $\mathcal{S}\in\mathcal{H}(\mathcal{C}_0)$ under the area budget $B$, applies a temporary counterfactual mask $M_{\mathcal{S}}$, and measures
\begin{equation}
\Delta(\mathcal{S}) = R(\mathbf{X}) - R(M_{\mathcal{S}}(\mathbf{X})).
\label{eq:score_drop}
\end{equation}
A valid trigger support should reduce the same mechanism score used for detection while occupying little area, so the final support is selected by
\begin{equation}
\mathcal{S}^{\star}
= \operatorname{ParetoElbow}\left\{ \big(|\mathcal{S}|,\Delta(\mathcal{S})\big): \mathcal{S}\in \mathcal{H}(\mathcal{C}_0) \right\} ,
\label{eq:pareto_support}
\end{equation}
where $\operatorname{ParetoElbow}$ keeps the area--drop frontier and maximizes drop gain minus area cost, with ties favoring compact supports. Candidates that fail $R(M_{\mathcal{S}}(\mathbf{X}))\leq\tau_{\mathrm{cal}}$ in Equation~\ref{eq:ccf_accept} are rejected; if none pass, the system enters fail-safe. Visual-density core search and closure refinements use image-grid connectivity and local score-drop similarity; Appendix~\ref{app:impl} gives implementation.

\paragraph{Recovery by localized inpainting.}
Once $\mathcal{S}^{\star}$ is accepted under Definition~\ref{def:ccf}, the observation is recovered by
\begin{equation}
\mathbf{X}_{\text{def}} = \operatorname{Inpaint}\big(\mathbf{X}_{\text{original}}, \mathcal{S}^{\star}\big).
\end{equation}
Inpainting is the default operator because it removes the causal support while preserving surrounding scene structure; local-mean, zero, and blur masking are treated as recovery ablations.

\paragraph{Per-observation routing and complexity.}
Each observation is routed to pass-through, localized recovery, or fail-safe. Clean frames pay only the detection cost; candidate search, counterfactual masking, and inpainting run only after the gate fires. Appendix~Table~\ref{tab:app_runtime_diagnostics} reports runtime accounting.\par

\section{Experiments}
Our experiments audit clean-preserving recovery, clean-calibrated detection, localization, baselines, and cross-architecture transfer; the appendix provides ablations, intervals, sweeps, runtime counters, oracle studies, and failure cases.

\subsection{Setup and Metrics}
We evaluate OpenVLA~\citep{kim2024openvlao} on LIBERO Spatial/Object/Goal/LIBERO-10~\citep{LIBERO2023} and test $\pi_{0.5}$~\citep{intelligence2025pi} on LIBERO-style and REAL transfer tasks. Training and large-scale rollout evaluation use an eight-GPU NVIDIA A800 machine. Attack reproduction follows the public BadVLA and INFUSE protocols, with defense code layered only at inference time. Each LIBERO cell uses 10 tasks with 50 rollouts per task. Main rows require paired 500-episode clean/trigger logs under the same checkpoint and frozen defense configuration; incomplete pairs appear only as appendix diagnostics.

\textbf{Attacks and baselines.}
We evaluate BadVLA objective-decoupled poisoning and INFUSE module injection after Stage-II clean fine-tuning. Baselines cover input preprocessing (JPEG $q=20$, Gaussian noise $\epsilon=0.08$) and parameter repair ($\Delta W$ Auditing, reset ratio $r=20\%$). Fine-Pruning remains in Appendix~Table~\ref{tab:app_fine_pruning_diagnostics} because available runs either lack paired 500-episode logs or collapse clean competence. Unlike whole-input or weight-repair baselines, \net{} performs observation-conditional monitoring and localized recovery; reproduction details are in Appendix~\ref{app:baseline_details}.

\begin{table*}[t]
\centering
\caption{OpenVLA/LIBERO defense results. Cells report SR(w/o) / SR(w) / VLA-ASR (\%); deltas denote VLA-ASR change from no defense. \textbf{Best} and \underline{second-best} VLA-ASR per attack are highlighted. Avg reports the unweighted mean VLA-ASR across the four LIBERO suites.}
\label{tab:main_results}
\setlength{\tabcolsep}{2.5pt}
\renewcommand{\arraystretch}{1.05}
\footnotesize
\resizebox{\textwidth}{!}{%
\begin{tabular}{ll|ccc|ccc|ccc|ccc|c}
\toprule
\multirow{2}{*}{\textbf{Attack}} & \multirow{2}{*}{\textbf{Defense}} 
& \multicolumn{3}{c|}{\textbf{LIBERO-Spatial}} 
& \multicolumn{3}{c|}{\textbf{LIBERO-Object}} 
& \multicolumn{3}{c|}{\textbf{LIBERO-Goal}} 
& \multicolumn{3}{c|}{\textbf{LIBERO-10}} 
& \textbf{Avg} \\
\cmidrule(lr){3-5} \cmidrule(lr){6-8} \cmidrule(lr){9-11} \cmidrule(lr){12-14}
& & SR(w/o)$\uparrow$ & SR(w)$\uparrow$ & VLA-ASR$\downarrow$ 
  & SR(w/o)$\uparrow$ & SR(w)$\uparrow$ & VLA-ASR$\downarrow$ 
  & SR(w/o)$\uparrow$ & SR(w)$\uparrow$ & VLA-ASR$\downarrow$ 
  & SR(w/o)$\uparrow$ & SR(w)$\uparrow$ & VLA-ASR$\downarrow$ 
  & VLA-ASR$\downarrow$ \\
\midrule

\rowcolor{gray!8}
\multicolumn{15}{l}{\textbf{Attack: BadVLA}} \\

  & No defense
    & 97.8 & 0.0  & 100.0
    & 98.6 & 0.0  & 100.0
    & 98.0 & 0.0  & 100.0
    & 94.0 & 0.0  & 100.0
    & 100.0\% \\

  & JPEG ($q{=}20$)
    & 96.4 & 0.0 & 100.0\,\textcolor{gray}{$\pm$0.0}
    & 98.6 & 0.0 & 100.0\,\textcolor{gray}{$\pm$0.0}
    & 97.6 & 0.0 & 100.0\,\textcolor{gray}{$\pm$0.0}
    & 91.4 & 0.0 & 100.0\,\textcolor{gray}{$\pm$0.0}
    & 100.0\% \\

  & Gaussian ($\epsilon{=}0.08$)
    & 96.0 & 0.0 & 100.0\,\textcolor{gray}{$\pm$0.0}
    & 98.4 & 0.0 & 100.0\,\textcolor{gray}{$\pm$0.0}
    & 97.6 & 0.0 & 100.0\,\textcolor{gray}{$\pm$0.0}
    & 85.0 & 0.0 & 100.0\,\textcolor{gray}{$\pm$0.0}
    & 100.0\% \\

  & $\Delta W$ Auditing ($r{=}20\%$)
    & 98.0 & 0.0 & 100.0\,\textcolor{gray}{$\pm$0.0}
    & 97.6 & 0.0 & 100.0\,\textcolor{gray}{$\pm$0.0}
    & 99.6 & 0.0 & 100.0\,\textcolor{gray}{$\pm$0.0}
    & 90.4 & 0.4 & \underline{99.6}\,\textcolor{green!50!black}{$\downarrow$0.4}
    & 99.9\% \\

  & \textbf{\net{} (Ours)}
    & \textbf{98.2} & \textbf{97.6} & \textbf{0.6}\,\textcolor{green!50!black}{$\downarrow$99.4}
    & \textbf{99.0} & \textbf{90.6} & \textbf{8.5}\,\textcolor{green!50!black}{$\downarrow$91.5}
    & \textbf{97.8} & \textbf{88.6} & \textbf{9.4}\,\textcolor{green!50!black}{$\downarrow$90.6}
    & \textbf{92.8} & \textbf{84.0} & \textbf{9.5}\,\textcolor{green!50!black}{$\downarrow$90.5}
    & \cellcolor{gray!15}\textbf{7.0\%}\,\textcolor{green!50!black}{$\downarrow$93.0\%} \\

\midrule

\rowcolor{gray!8}
\multicolumn{15}{l}{\textbf{Attack: INFUSE}} \\

  & No defense
    & 97.6 & 0.0  & 100.0
    & 98.0 & 0.0  & 100.0
    & 95.8 & 0.0  & 100.0
    & 92.6 & 0.0  & 100.0
    & 100.0\% \\

  & JPEG ($q{=}20$)
    & 97.6 & 0.0 & 100.0\,\textcolor{gray}{$\pm$0.0}
    & 97.4 & 0.0 & 100.0\,\textcolor{gray}{$\pm$0.0}
    & 94.4 & 0.0 & 100.0\,\textcolor{gray}{$\pm$0.0}
    & 87.8 & 0.0 & 100.0\,\textcolor{gray}{$\pm$0.0}
    & 100.0\% \\

  & Gaussian ($\epsilon{=}0.08$)
    & 97.8 & 0.0 & 100.0\,\textcolor{gray}{$\pm$0.0}
    & 97.0 & 0.0 & 100.0\,\textcolor{gray}{$\pm$0.0}
    & 94.0 & 0.0 & 100.0\,\textcolor{gray}{$\pm$0.0}
    & 85.0 & 0.0 & 100.0\,\textcolor{gray}{$\pm$0.0}
    & 100.0\% \\

  & $\Delta W$ Auditing ($r{=}20\%$)
    & 97.6 & 77.2 & \underline{20.9}\,\textcolor{green!50!black}{$\downarrow$79.1}
    & 98.0 & 61.4 & \underline{37.3}\,\textcolor{green!50!black}{$\downarrow$62.7}
    & 92.6 & 1.0 & 98.9\,\textcolor{green!50!black}{$\downarrow$1.1}
    & 89.2 & 15.0 & \underline{83.2}\,\textcolor{green!50!black}{$\downarrow$16.8}
    & \underline{60.1\%} \\

  & \textbf{\net{} (Ours)}
    & \textbf{97.4} & \textbf{97.4} & \textbf{0.0}\,\textcolor{green!50!black}{$\downarrow$100.0}
    & \textbf{98.0} & \textbf{97.4} & \textbf{0.6}\,\textcolor{green!50!black}{$\downarrow$99.4}
    & \textbf{95.0} & \textbf{93.4} & \textbf{1.7}\,\textcolor{green!50!black}{$\downarrow$98.3}
    & \textbf{92.4} & \textbf{86.6} & \textbf{6.3}\,\textcolor{green!50!black}{$\downarrow$93.7}
    & \cellcolor{gray!15}\textbf{2.2\%}\,\textcolor{green!50!black}{$\downarrow$97.8\%} \\

\bottomrule
\end{tabular}
}
\end{table*}

\textbf{Metrics.}
VLA-ASR measures residual backdoor effectiveness as triggered-task degradation:
\begin{equation}
\text{VLA-ASR} = \max\left(0,\ \frac{SR^{(w/o)} - SR^{(w)}}{SR^{(w/o)} + \varepsilon}\right) \times 100\%,
\end{equation}
where $SR^{(w/o)}$ and $SR^{(w)}$ are defended success rates without and with the trigger, and $\varepsilon = 10^{-6}$. Clipping handles rollout noise; VLA-ASR $=0$ means triggered performance matches the trigger-free level, and $100\%$ means all trigger-free success is lost. Detection rate is $\text{DR}=TP/(TP+FN)$, and clean false alarm rate is $\text{FAR}=FP/(FP+TN)$. Reporting SR(w/o), SR(w), VLA-ASR, DR, and FAR separates detection from recovery; Appendix~\ref{app:impl} reports Wilson intervals.
Unlike classification-style ASR, which measures the fraction of inputs mapped to an attacker's target label, VLA-ASR measures task-level degradation relative to the defended trigger-free baseline. This formulation is appropriate for long-horizon control, where many distinct failure trajectories can constitute successful attacks and no single target action need be specified.

\subsection{Main Defense Results}
\label{sec:main_results}
Table~\ref{tab:main_results} reports paired 500-episode cells for the final configuration; earlier variants and targeted subsets remain appendix diagnostics. Standard defenses leave high residual VLA-ASR: for BadVLA, no defense, JPEG, Gaussian, and $\Delta W$ Auditing retain 99.9--100.0\% average VLA-ASR, while \net{} reduces VLA-ASR to 0.6\%, 8.5\%, 9.4\%, and 9.5\% on the four LIBERO suites. This pattern is important because the recovery is not achieved by globally suppressing the policy: \net{} keeps trigger-free SR close to the undefended clean row and pulls triggered rollouts back toward that same no-trigger baseline. INFUSE also transfers after clean Stage-II fine-tuning, with final triggered SR of 97.4\%, 97.4\%, 93.4\%, and 86.6\%, indicating that the monitored mechanism persists beyond the BadVLA optimization path rather than overfitting to one poisoning recipe.

\par\vspace{-2pt}
\noindent\textbf{Detection reliability.}
Table~\ref{tab:detection_reliability} separates detection from recovery. Both attacks show 0/2000 clean false alarms. BadVLA detection misses are confined to Goal (489/500 DR); LIBERO-10 reaches 500/500 detections, so its residual gap is attributable to localization and recovery rather than detection.

\begin{table}[t]
\centering
\caption{\textbf{Detection reliability.} Frozen clean-calibrated thresholds; episodes / total = percent.}
\label{tab:detection_reliability}
\scriptsize
\setlength{\tabcolsep}{4pt}
\renewcommand{\arraystretch}{1.08}
\resizebox{0.82\linewidth}{!}{%
\begin{tabular}{lcc|cc}
\toprule
\textbf{Suite} & \multicolumn{2}{c|}{\textbf{BadVLA}} & \multicolumn{2}{c}{\textbf{INFUSE}} \\
\cmidrule(lr){2-3}\cmidrule(lr){4-5}
& \textbf{FAR}$\downarrow$ & \textbf{DR}$\uparrow$ & \textbf{FAR}$\downarrow$ & \textbf{DR}$\uparrow$ \\
\midrule
Spatial & 0/500 = 0.0 & 500/500 = 100.0 & 0/500 = 0.0 & 500/500 = 100.0 \\
Object  & 0/500 = 0.0 & 500/500 = 100.0 & 0/500 = 0.0 & 500/500 = 100.0 \\
Goal    & 0/500 = 0.0 & 489/500 = 97.8  & 0/500 = 0.0 & 500/500 = 100.0 \\
LIBERO-10 & 0/500 = 0.0 & 500/500 = 100.0 & 0/500 = 0.0 & 500/500 = 100.0 \\
\midrule
\rowcolor{gray!10}
Total & 0/2000 = 0.0 & 1989/2000 = 99.5 & 0/2000 = 0.0 & 2000/2000 = 100.0 \\
\bottomrule
\end{tabular}%
}
\end{table}

\begin{table}[t]
\centering
\caption{\textbf{Cross-architecture $\pi_{0.5}$ transfer.} Four LIBERO-style configurations; \net{} lowers VLA-ASR while recovering triggered task success.}
\label{tab:pi05_libero_results}
\footnotesize
\setlength{\tabcolsep}{4pt}
\resizebox{\linewidth}{!}{%
\begin{tabular}{lcccccccc}
\toprule
\textbf{Method} & \multicolumn{2}{c}{\textbf{Spatial}} & \multicolumn{2}{c}{\textbf{Object}} & \multicolumn{2}{c}{\textbf{Goal}} & \multicolumn{2}{c}{\textbf{LIBERO-10}}\\
\cmidrule(lr){2-3} \cmidrule(lr){4-5} \cmidrule(lr){6-7} \cmidrule(lr){8-9}
& SR(w) $\uparrow$ & VLA-ASR $\downarrow$ & SR(w) $\uparrow$ & VLA-ASR $\downarrow$ & SR(w) $\uparrow$ & VLA-ASR $\downarrow$ & SR(w) $\uparrow$ & VLA-ASR $\downarrow$\\
\midrule
Clean baseline          & 95.8 & --    & 98.4 & --    & 92.8 & --    & 88.4 & --    \\
BadVLA (no trigger)     & 91.2 & --    & 96.0 & --    & 93.2 & --    & 79.8 & --    \\
BadVLA (no defense)     & 0.2  & 99.78 & 0.2  & 99.79 & 0.6  & 99.36 & 0.0  & 100   \\
\textbf{\net{} (Ours)}
& \textbf{90.8} & \textbf{0.44}
& \textbf{90.2} & \textbf{6.04}
& \textbf{82.2} & \textbf{11.80}
& \textbf{60.6} & \textbf{24.06} \\
\bottomrule
\end{tabular}%
}
\end{table}

\FloatBarrier

\subsection{Robustness and Transfer}
\noindent
\textbf{Trigger-variant robustness (same backbone).}
Figure~\ref{fig:trigger_robustness} summarizes robustness to trigger appearance, size, and location. BadVLA remains highly effective across variants, while \net{} keeps residual VLA-ASR near 0--9\%, yielding absolute reductions of 86.0--100 percentage points. The size and position sweeps indicate that the detector does not memorize a fixed scale or center coordinate.

\begin{figure}[!ht]
\centering
\includegraphics[width=0.82\linewidth]{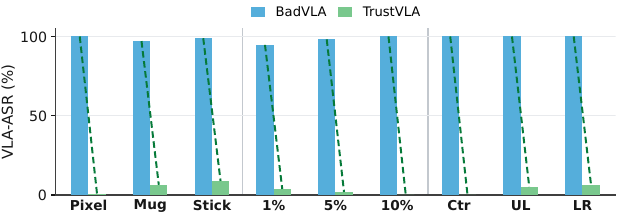}
\caption{\textbf{Robustness sweep.} Trigger appearance, size, and location variants; bars show VLA-ASR before and after \net{}. Lower is better.}
\label{fig:trigger_robustness}
\end{figure}

\noindent\textbf{Cross-architecture transfer ($\pi_{0.5}$).}
Beyond OpenVLA, Table~\ref{tab:pi05_libero_results} evaluates $\pi_{0.5}$ on four LIBERO-style transfer logs. Because $\pi_{0.5}$ uses a different action decoder, we reuse only architecture-independent monitor pieces: language-head evidence features, small clean calibration, and visual-grid localization/inpainting (details in Appendix~\ref{app:impl}). These rows are a transfer feasibility check rather than a tuned $\pi_{0.5}$ defense. Backdoored models keep high no-trigger SR but reach 99.36--100\% VLA-ASR; \net{} lowers VLA-ASR to 0.44--24.06\% and recovers triggered SR to 60.6--90.8\%.

\begin{figure*}[t]
\centering
\includegraphics[width=0.82\textwidth]{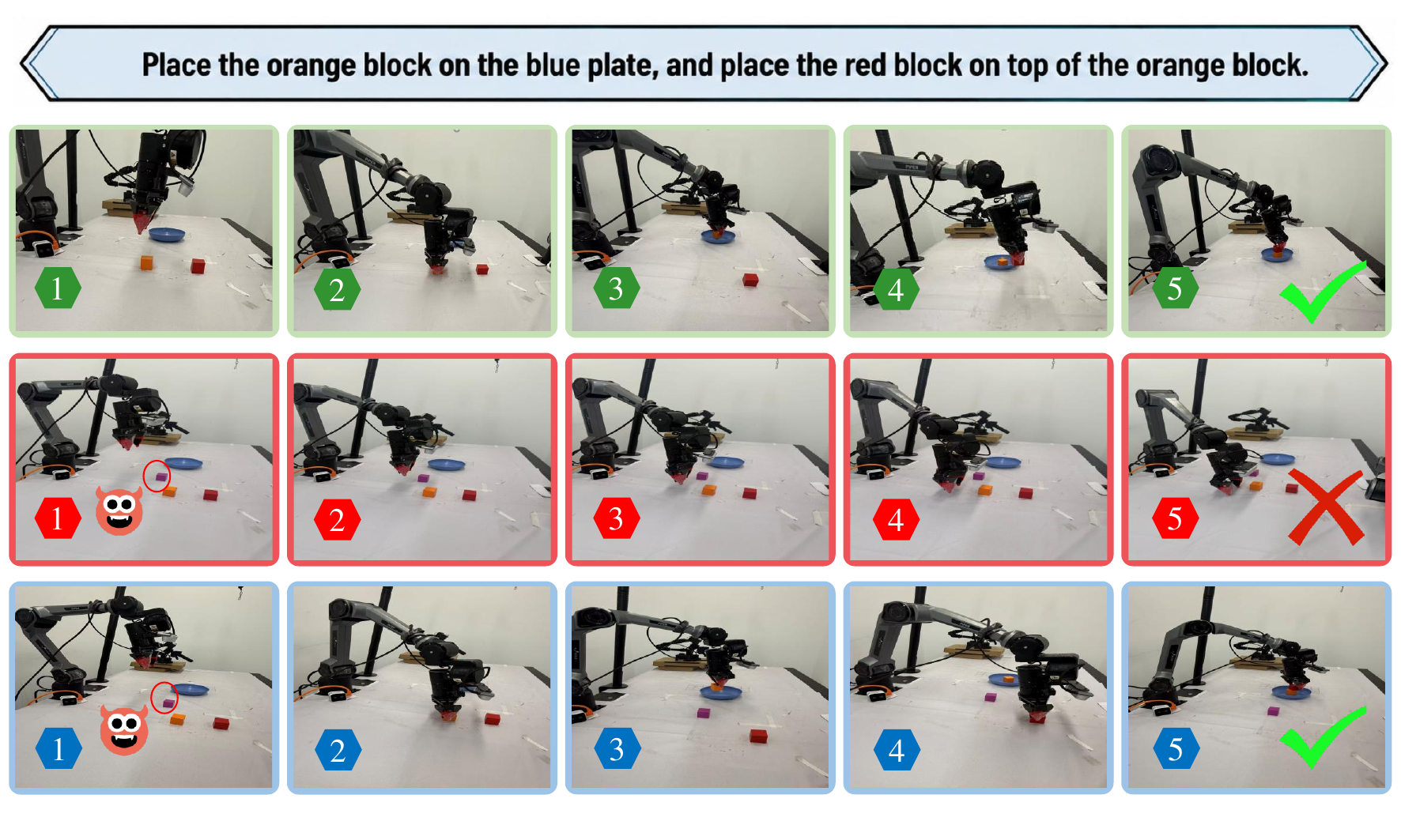}
\caption{\textbf{REAL visual-trigger rollout (Blocks task).} The trigger-free policy succeeds (top), the purple-trigger no-defense rollout fails (middle), and \net{} localizes/inpaints the trigger support to recover task-directed behavior (bottom).}
\label{fig:real_rollout}
\end{figure*}

\begin{table}[t]
\begin{minipage}[t]{0.37\linewidth}
\vspace{0pt}
\raggedright
\textbf{Physical-robot setup.}
We deploy the BadVLA-poisoned $\pi_{0.5}$ policy on a two-camera 6-DoF manipulator with a shared purple-block trigger; Appendix~\ref{app:physical_robot_setup} gives the full setup. Across 30 rollouts per task, triggers reduce both tasks to 1/30 success, while \net{} recovers 15/30 and 24/30. Figure~\ref{fig:real_rollout} shows a representative rollout.
\end{minipage}\hfill
\begin{minipage}[t]{0.60\linewidth}
\vspace{0pt}
\centering
\caption{$\pi_{0.5}$ REAL transfer. Physical-robot tasks; 30 rollouts per cell; ASR uses BadVLA clean as SR(w/o).}
\label{tab:real_results}
\scriptsize
\setlength{\tabcolsep}{4.5pt}
\renewcommand{\arraystretch}{0.98}
\begin{tabular}{@{}lcccc@{}}
\toprule
\textbf{Method} & \multicolumn{2}{c}{\begin{tabular}[c]{@{}c@{}}\textbf{Strawberries}\\[-1pt]\textit{all berries $\rightarrow$ bowl}\end{tabular}} & \multicolumn{2}{c}{\begin{tabular}[c]{@{}c@{}}\textbf{Blocks}\\[-1pt]\textit{plate then stack}\end{tabular}} \\
\cmidrule(lr){2-3} \cmidrule(lr){4-5}
& SR $\uparrow$ & ASR $\downarrow$ & SR $\uparrow$ & ASR $\downarrow$ \\
\midrule
Clean       & 26/30 & -- & 30/30 & --  \\
BadVLA clean  & 17/30 & -- & 28/30 & --  \\
BadVLA trig.  & 1/30  & 94.12 & 1/30 & 96.43  \\
\textbf{TrustVLA} & \textbf{15/30} & \textbf{11.76} & \textbf{24/30} & \textbf{14.29} \\
\bottomrule
\end{tabular}
\end{minipage}
\end{table}

\subsection{Ablations, Baselines, and Deployment}
Appendix~Tables~\ref{tab:app_ablation} and~\ref{tab:app_oracle_inpaint_upper_bound} show that localization is not merely attention: attention-only recovery is unstable, score-drop variants change \emph{which region} is masked, and high-drop cluster closure improves LIBERO-10 (\textbf{57.0 $\to$ 83.5}) while sometimes over-extending in Goal; remaining gaps separate localization from inpainting.
Table~\ref{tab:main_results} compares \net{} with JPEG, Gaussian noise, and $\Delta W$ Auditing; these retain high residual VLA-ASR, while Fine-Pruning collapses clean competence (\mbox{Appendix~Table~\ref{tab:app_fine_pruning_diagnostics}}). \net{} assumes self-hosted hidden-state access and a small clean calibration set; Appendix~\ref{app:impl} reports stress tests, anomaly diagnostics, and runtime.

\section{Conclusion}

We identify a compact causal footprint shared by BadVLA and INFUSE, and propose \net{}, a mechanism-guided inference-time defense that monitors per-token, per-layer Dirichlet evidence, validates compact supports by counterfactual score drop, and recovers observations via localized inpainting. Across OpenVLA/LIBERO and $\pi_{0.5}$ transfer evaluations, \net{} reduces VLA-ASR while preserving clean-task performance. More broadly, clean-calibrated mechanism scores with counterfactual support search may help diagnose other internal-state anomalies whenever clean rollouts define a normal operating region. \noindent\textbf{Limitation.}\label{sec:limitations}
\net{} targets \emph{visual-triggered} VLA backdoors with self-hosted hidden-state access; recovery remains conditional on the compact-footprint hypothesis. Semantic or non-visual triggers, severe distribution shift, and adaptive attackers that mimic evidence while suppressing localization remain outside our robustness claim (Appendix~\ref{app:adaptive_threats}).



\bibliographystyle{plainnat}
\bibliography{main}


\appendix

\section*{Appendix Overview}

In the supplemental material:
\begin{itemize}[leftmargin=1.5em,itemsep=1pt,topsep=1pt,parsep=0pt]
    \item \textbf{Appendix~\ref{app:theory}} provides the stylized theory behind epistemic homogenization, rank promotion, compact-support recovery, and adaptive-threat diagnostics.
    \item \textbf{Appendix~\ref{app:impl}} provides implementation details, small clean calibration, confidence intervals, hyperparameters, and runtime accounting.
    \item \textbf{Appendix~\ref{app:ablation_protocol}} reports component ablations, the oracle inpainting upper bound, and additional localization evidence.
    \item \textbf{Appendix~\ref{app:baseline_details}} documents baseline reproduction details, including Fine-Pruning diagnostics.
    \item \textbf{Appendix~\ref{app:dw_audit}} explains how $\Delta W$ Auditing is used as a checkpoint-repair baseline.
    \item \textbf{Appendix~\ref{app:cross_attack}} verifies the mechanism on INFUSE and reports post-clean-fine-tuning evaluations.
    \item \textbf{Appendix~\ref{app:failure_modes}} separates detection, localization, recovery, and non-localizable-trigger failure modes.
    \item \textbf{Appendix~\ref{app:supp_run_visuals}} includes qualitative rollout visualizations.
\end{itemize}

\section{Theoretical Analysis}
\label{app:theory}

We provide stylized propositions and a recovery theorem that formalize the empirical mechanism introduced in Section~\ref{evidence}, and connect them to the diagnostics used by \net{}. The analysis abstracts a VLA policy into token-level hidden states, evidence scores, attention logits, and a Lipschitz action head. The goal is to explain why the empirical features used by \net{} are expected under a persistent trigger-induced representation shift, not to characterize all VLA architectures.

\paragraph{Formal setup.}
Let $\{\mathbf{h}_i\}_{i=1}^{N}$ denote image-token hidden states before action decoding. For clean inputs, write $\mathbf{h}_i=\mathbf{s}_i+\boldsymbol{\xi}_i$, where $\mathbf{s}_i$ is spatially varying task semantics and $\boldsymbol{\xi}_i$ is bounded representation noise. A triggered input adds a shared direction,
\begin{equation}
\mathbf{h}^{\mathrm{bd}}_i = \mathbf{s}_i + \rho \mathbf{b} + \boldsymbol{\xi}_i ,
\end{equation}
where $\mathbf{b}$ is the backdoor direction and $\rho>0$ is its strength. Symbols are summarized in Table~\ref{tab:app_notation}.

For a fixed layer $l$ and visual token $i$, the per-token total evidence
$E_i=S_i^{(l)}-V=\sum_{k=1}^{V}\exp(\mathbf{w}_k^\top\mathbf{h}_i^{(l)})$ is the accumulated evidence after removing the unit Dirichlet prior from the concentration mass. The corresponding epistemic uncertainty is $\mathrm{EU}_i=V/(S_i^{(l)}+V)=V/(E_i+2V)$, which is equivalent to the per-token form of Eq.~(2) under the substitution $S=E+V$. The layer superscript is omitted when clear from context.

\begin{assumption}[Shared evidence shift]
For trigger-influenced visual tokens, the backdoor direction induces a common positive shift in the accumulated evidence:
\begin{equation}
E^{\mathrm{bd}}_i = c(\rho)E_i+r_i,\qquad c(\rho)>1,\quad |r_i|\leq\delta_E ,
\end{equation}
where $c(\rho)$ increases with trigger strength and $\delta_E$ captures residual token-specific variation.
\end{assumption}

\begin{proposition}[Backdoor-induced epistemic homogenization]\label{prop:eu_homogenization}
Under the shared evidence shift assumption, the spatial variance of epistemic uncertainty $\mathrm{EU}_i=V/(E_i+2V)$ decreases as the backdoor evidence strength increases:
\begin{equation}
\mathrm{Var}_i\!\left[\mathrm{EU}^{\mathrm{bd}}_i\right]
\leq
\frac{V^2}{\left(c(\rho)E_{\min}+2V-\delta_E\right)^4}
\mathrm{Var}_i[c(\rho)E_i+r_i],
\end{equation}
and therefore decreases on the order of $O(c(\rho)^{-2})$ when residual variation scales with clean evidence.
\end{proposition}

\begin{proposition}[Deep-layer rank promotion]\label{prop:rank_promotion}
Let $\Delta_m^{\mathrm{deep}}=\sigma_{(m)}^{\mathrm{deep}}-\sigma_t^{\mathrm{deep}}$ be the margin between the trigger token's clean deep-layer attention score and the $m$-th largest clean deep-layer attention score. If the deep-layer backdoor bias satisfies $\bar{\beta}_{\mathrm{deep}}>\Delta_m^{\mathrm{deep}}$, while its all-layer average bias remains below the corresponding global margin, then the trigger token enters the deep top-$m$ set but remains outside the all-layer top-$m$ set:
\begin{equation}
t\in \mathcal{T}_{\mathrm{deep}}
\quad\text{and}\quad
t\notin \mathcal{T}_{\mathrm{all}} .
\end{equation}
\end{proposition}

Let $\phi(x)$ denote the visual representation passed to the action head $g$, i.e., $\hat{\mathbf{a}}=g(\phi(x))$.

\begin{theorem}[Detection, localization, and recovery under margins]\label{thm:trustvla_margin}
Assume clean and triggered episodes are separated by a positive margin in the mechanism score, and trigger tokens satisfy the rank-promotion margin in Proposition~\ref{prop:rank_promotion}. Suppose there exists a compact support $\mathcal{S}^{\star}$ in the candidate pool, with $|\mathcal{S}^{\star}|\leq B$, such that counterfactual masking returns the score to the clean-calibrated region, $R(M_{\mathcal{S}^{\star}}(x_{\mathrm{bd}}))\leq\tau_{\mathrm{cal}}$. Let $I_{\mathcal{S}^{\star}}$ denote the localized recovery operator, including inpainting, applied to that support. If the recovered visual representation has residual error
\begin{equation}
\|\phi(I_{\mathcal{S}^{\star}}(x_{\mathrm{bd}}))-\phi(x_{\mathrm{clean}})\|_2\leq \eta ,
\end{equation}
and the policy head is $L$-Lipschitz in the visual representation, then \net{} detects the triggered input, accepts $\mathcal{S}^{\star}$ as an operational support, and the defended action satisfies
\begin{equation}
\|\hat{\mathbf{a}}_{\mathrm{def}}-\hat{\mathbf{a}}_{\mathrm{clean}}\|_2
\leq L\eta .
\end{equation}
\end{theorem}

\paragraph{Remark on continuous actions.}
The $L$-Lipschitz assumption applies to the continuous representation and pre-discretization action output used in the theorem. It does not by itself certify invariance of discrete action tokens after an argmax or sampling step: a small logit perturbation can still flip a token near a decision boundary. The success rates in Table~\ref{tab:main_results} should therefore be read as empirical evidence that the residual $\eta$ is usually small enough to avoid catastrophic action flips, not as a tight certified consequence of the Lipschitz bound.

\begin{proof}[Proof of Proposition~\ref{prop:eu_homogenization}]
Let $f(x)=V/(x+2V)$. For any evidence value $x\geq x_{\min}$,
\begin{equation}
|f'(x)| = \frac{V}{(x+2V)^2} \leq \frac{V}{(x_{\min}+2V)^2}.
\end{equation}
Triggered evidence satisfies $E_i^{\mathrm{bd}}=c(\rho)E_i+r_i\geq c(\rho)E_{\min}-\delta_E$. By the mean-value theorem,
\begin{equation}
|\mathrm{EU}_i^{\mathrm{bd}}-\mathrm{EU}_j^{\mathrm{bd}}|
\leq
\frac{V}{(c(\rho)E_{\min}+2V-\delta_E)^2}
|E_i^{\mathrm{bd}}-E_j^{\mathrm{bd}}|.
\end{equation}
Squaring and averaging over token pairs yields the variance bound. When residual variation is bounded or scales no faster than the clean evidence variation, the denominator grows faster than the numerator as $c(\rho)$ increases, so token-wise uncertainty becomes spatially homogenized.
\end{proof}

\begin{proof}[Proof of Proposition~\ref{prop:rank_promotion}]
Let $\sigma_t^{\mathrm{deep}}$ be the trigger token's clean deep-layer attention score and let $\sigma_{(m)}^{\mathrm{deep}}$ be the $m$-th largest clean score among non-trigger tokens. The clean margin is $\Delta_m^{\mathrm{deep}} = \sigma_{(m)}^{\mathrm{deep}} - \sigma_t^{\mathrm{deep}}$. If $\bar{\beta}_{\mathrm{deep}}>\Delta_m^{\mathrm{deep}}$, then $\sigma_t^{\mathrm{deep}}+\bar{\beta}_{\mathrm{deep}} > \sigma_{(m)}^{\mathrm{deep}}$, so the trigger token must enter the deep top-$m$ set. If the all-layer average bias is below the analogous all-layer margin, the same token remains outside the all-layer top-$m$ set. Therefore the trigger token lies in $\mathcal{T}_{\mathrm{deep}}\setminus\mathcal{T}_{\mathrm{all}}$, placing it in the candidate pool used by counterfactual support search.
\end{proof}

\begin{proof}[Proof of Theorem~\ref{thm:trustvla_margin}]
Margin separation means there exists a threshold $\tau_{\mathrm{cal}}$ and margin $\gamma>0$ such that clean inputs satisfy $R(\mathbf{X})\leq \tau_{\mathrm{cal}}-\gamma$ and triggered inputs satisfy $R(\mathbf{X})\geq \tau_{\mathrm{cal}}+\gamma$, so thresholding at $\tau_{\mathrm{cal}}$ separates the two cases. Proposition~\ref{prop:rank_promotion} ensures that trigger tokens enter the candidate pool when the deep-layer bias margin holds. By assumption, a compact support $\mathcal{S}^{\star}$ returns the score to the clean-calibrated region and therefore passes the score-restoration acceptance rule; among accepted supports, the Pareto rule selects a compact high-drop support.

The recovered observation is $I_{\mathcal{S}^{\star}}(x_{\mathrm{bd}})$, not merely the masked image; the residual $\eta$ includes both any remaining trigger perturbation and any inpainting distortion. By assumption, $\|\phi(I_{\mathcal{S}^{\star}}(x_{\mathrm{bd}}))-\phi(x_{\mathrm{clean}})\|_2\leq \eta$. By $L$-Lipschitzness of $g$,
\begin{equation}
\|\hat{\mathbf{a}}_{\mathrm{def}}-\hat{\mathbf{a}}_{\mathrm{clean}}\|_2
=
\|g(\phi(I_{\mathcal{S}^{\star}}(x_{\mathrm{bd}})))-g(\phi(x_{\mathrm{clean}}))\|_2
\leq
L\eta .
\end{equation}
Under detection and localization margins, the defended action remains close to the clean action.
\end{proof}

\subsection{Mechanistic Interpretation and Adaptive Threats}
\label{app:mechanism_formulae}
\label{app:adaptive_threats}

The propositions above give margin statements. We add three interpretive points that read across the diagnostics, then frame what an adaptive attack would have to preserve.

The Dirichlet view separates accumulated evidence from residual epistemic uncertainty. With $\alpha_k=e_k+1$ and $S=\sum_k\alpha_k=E+V$, the normalized preference $p_k=\alpha_k/S$ can be sharp while $\mathrm{EU}=V/(S+V)=V/(E+2V)$ still records whether the model has gathered broad evidence. A backdoor shortcut can increase $S$ through a shared evidence mode without preserving the token-dependent residuals needed for manipulation, so Dirichlet uncertainty tests whether the evidence supporting an action remains spatially structured the way clean rollouts predict, rather than certifying the action itself. The same shortcut model predicts low effective rank in the triggered token cloud: writing $\mathbf{H}^{(l),\mathrm{bd}} \approx \mathbf{H}^{(l),\mathrm{clean}} + z(x)\rho_l\mathbf{s}(\mathbf{b}^{(l)})^\top$ adds a rank-one component, so the effective rank $r_{\mathrm{eff}}(\mathbf{H})=\exp(-\sum_j p_j\log p_j)$ with $p_j=\sigma_j^2/\sum_k\sigma_k^2$ decreases when the trigger magnitude dominates local scene variation. Counterfactual score drop turns this signal into a localization criterion: a simple additive approximation yields
\begin{equation}
R(\mathbf{X}) - R(M_{\mathcal{S}}(\mathbf{X})) \approx \rho\sum_{i\in \mathcal{S}\cap\mathcal{S}^{\star}} q_i - \lambda_{\mathrm{ctx}}|\mathcal{S}\setminus\mathcal{S}^{\star}|,
\end{equation}
which explains why attention-only can fail (salient task objects need not reduce $R$), max-drop can over-select destructive regions, and Pareto/closure selection favors compact supports that suppress the mechanism while preserving task context.

\begin{figure*}[t]
\centering
\setlength{\tabcolsep}{3pt}
\begin{tabular}{ccc}
\includegraphics[width=0.31\textwidth]{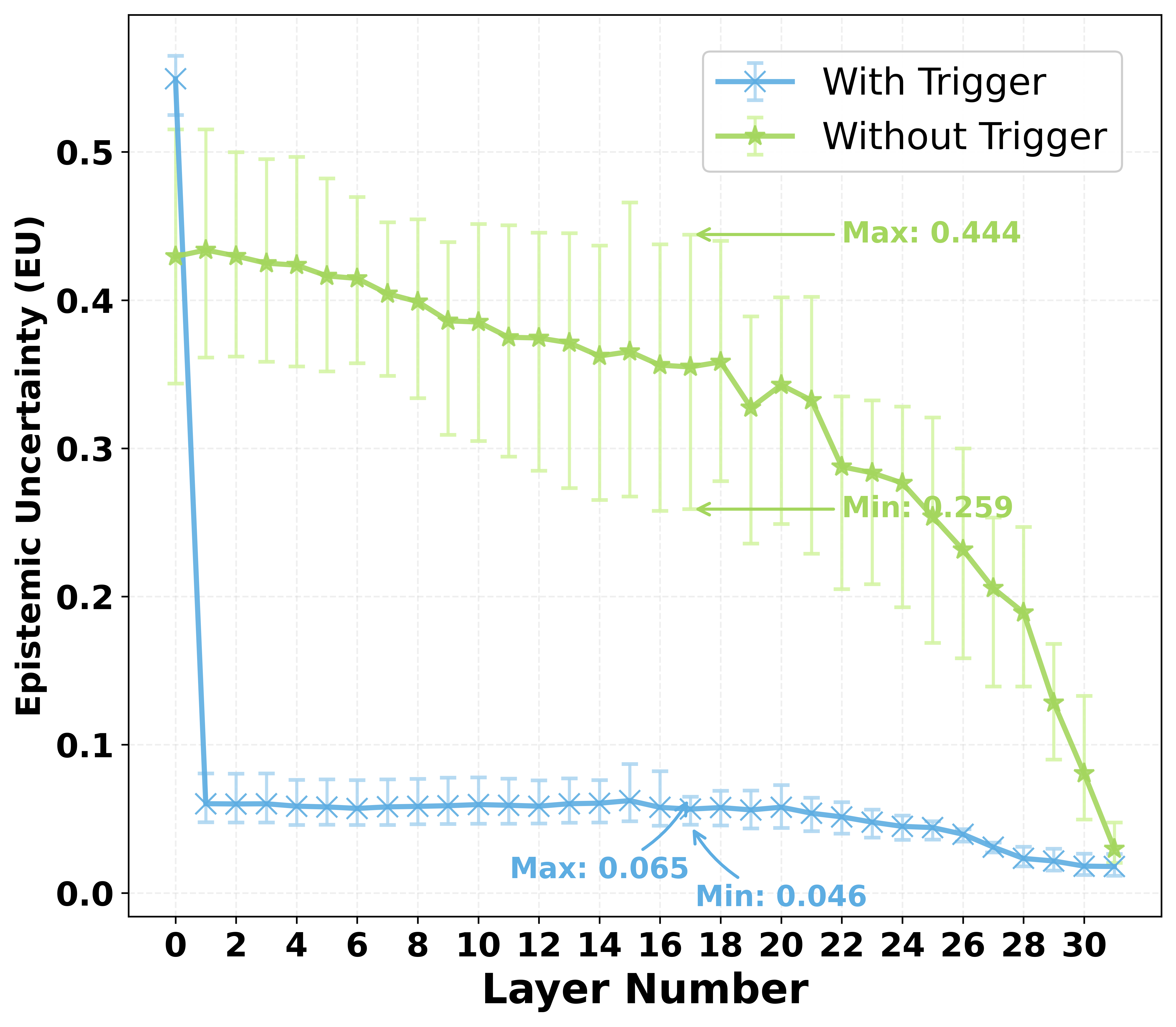} &
\includegraphics[width=0.31\textwidth]{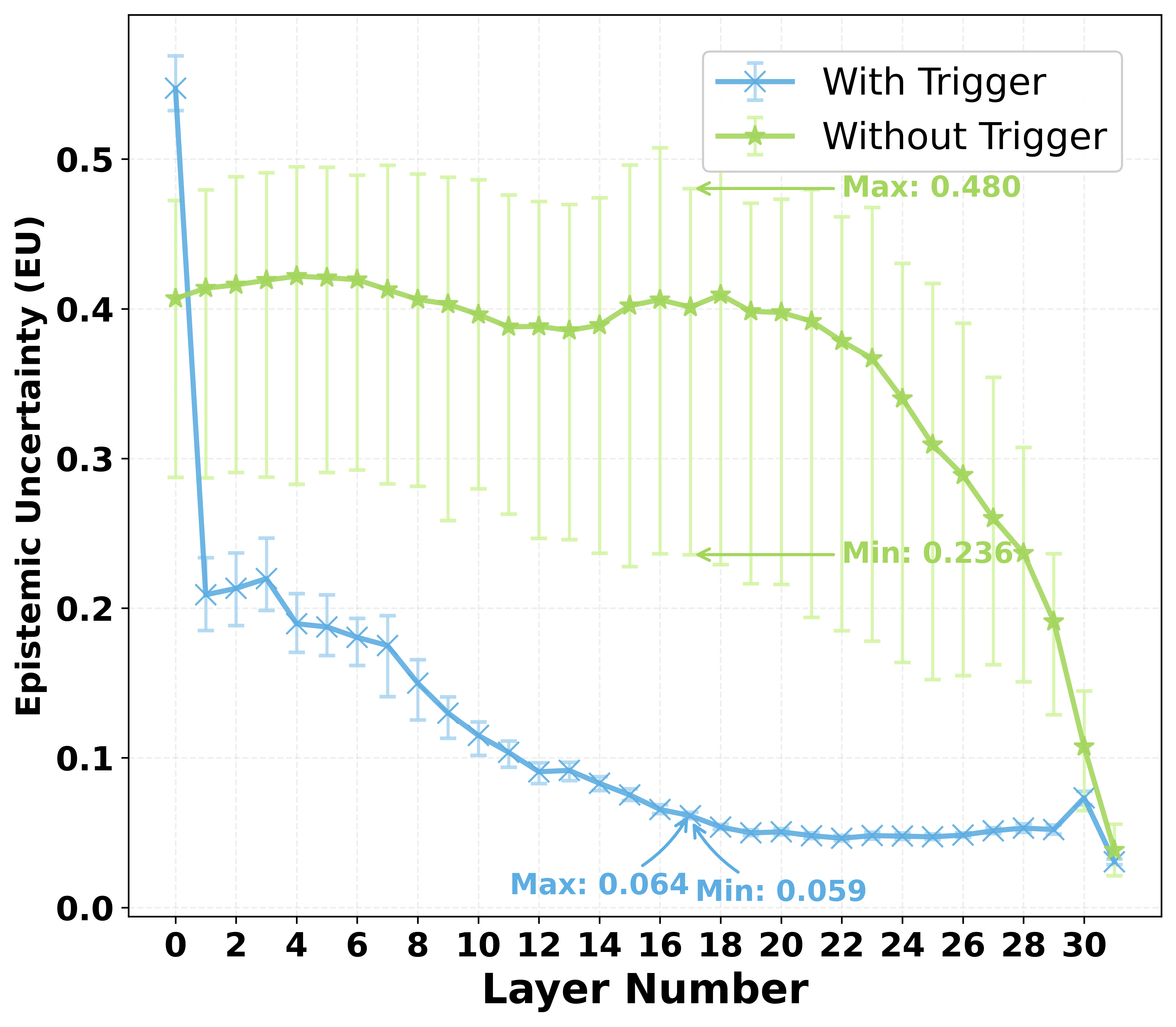} &
\includegraphics[width=0.31\textwidth]{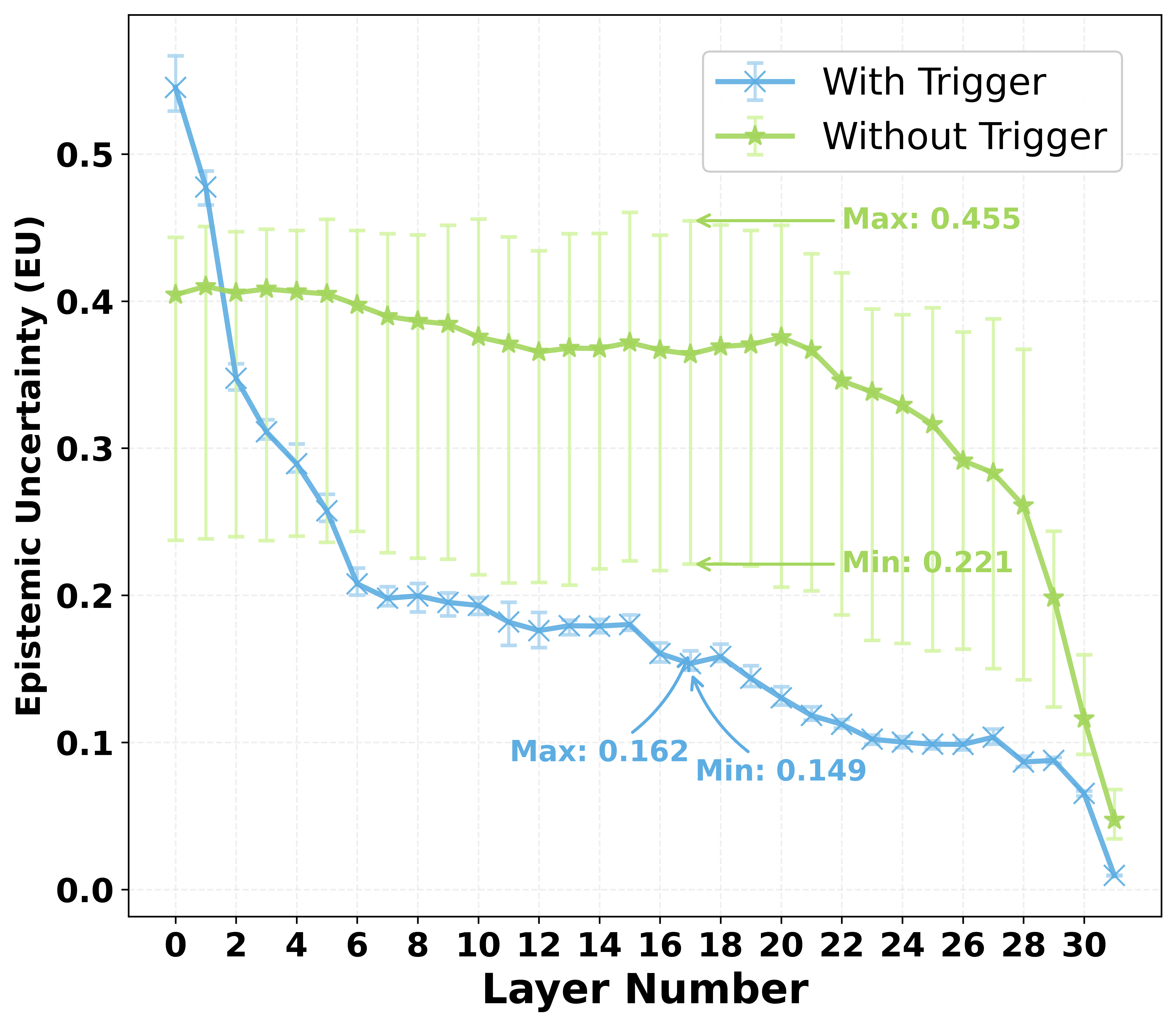} \\
{\scriptsize (a) LIBERO Object} &
{\scriptsize (b) LIBERO Spatial} &
{\scriptsize (c) LIBERO Goal} \\[3pt]
\includegraphics[width=0.31\textwidth]{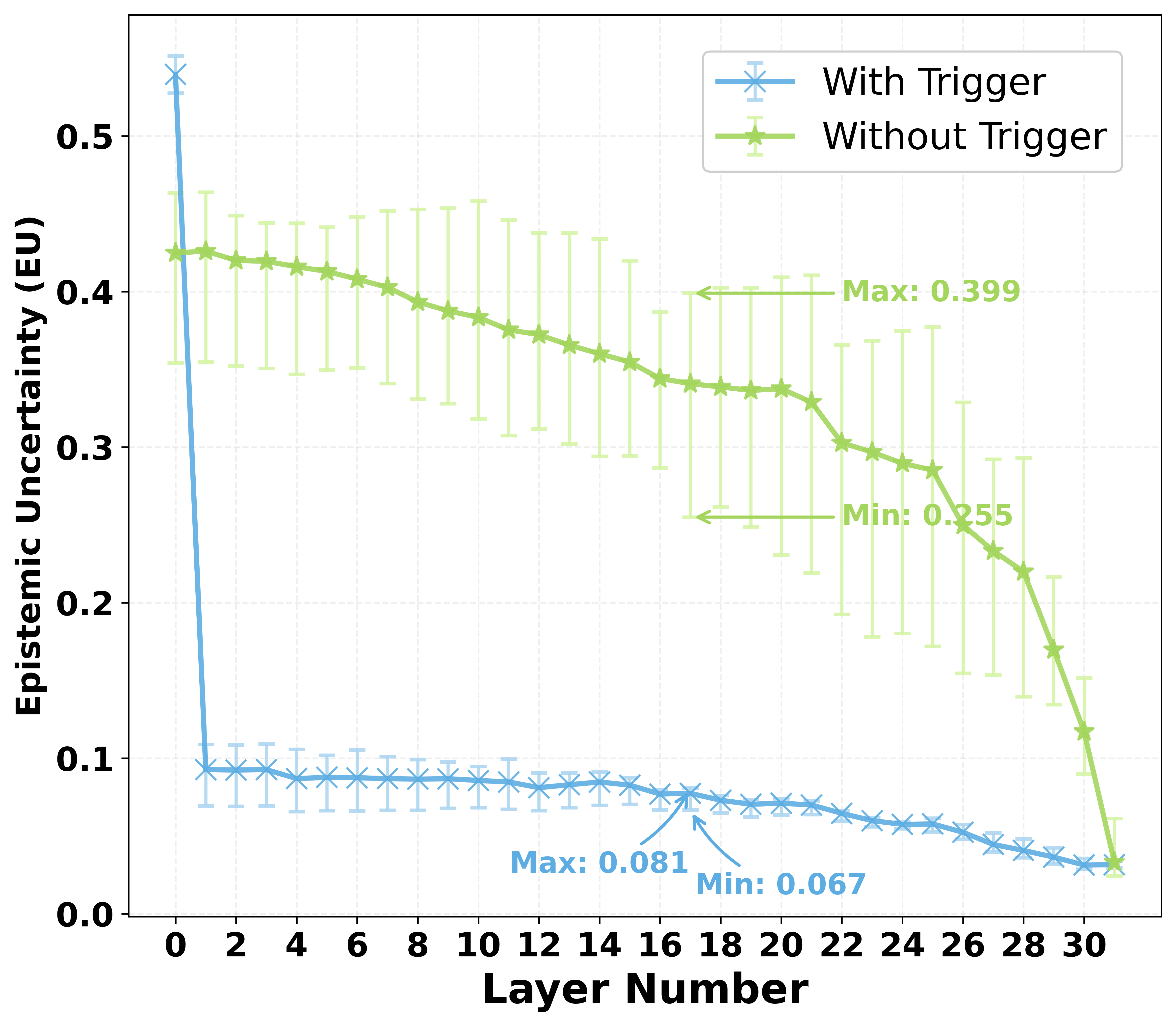} &
\includegraphics[width=0.31\textwidth]{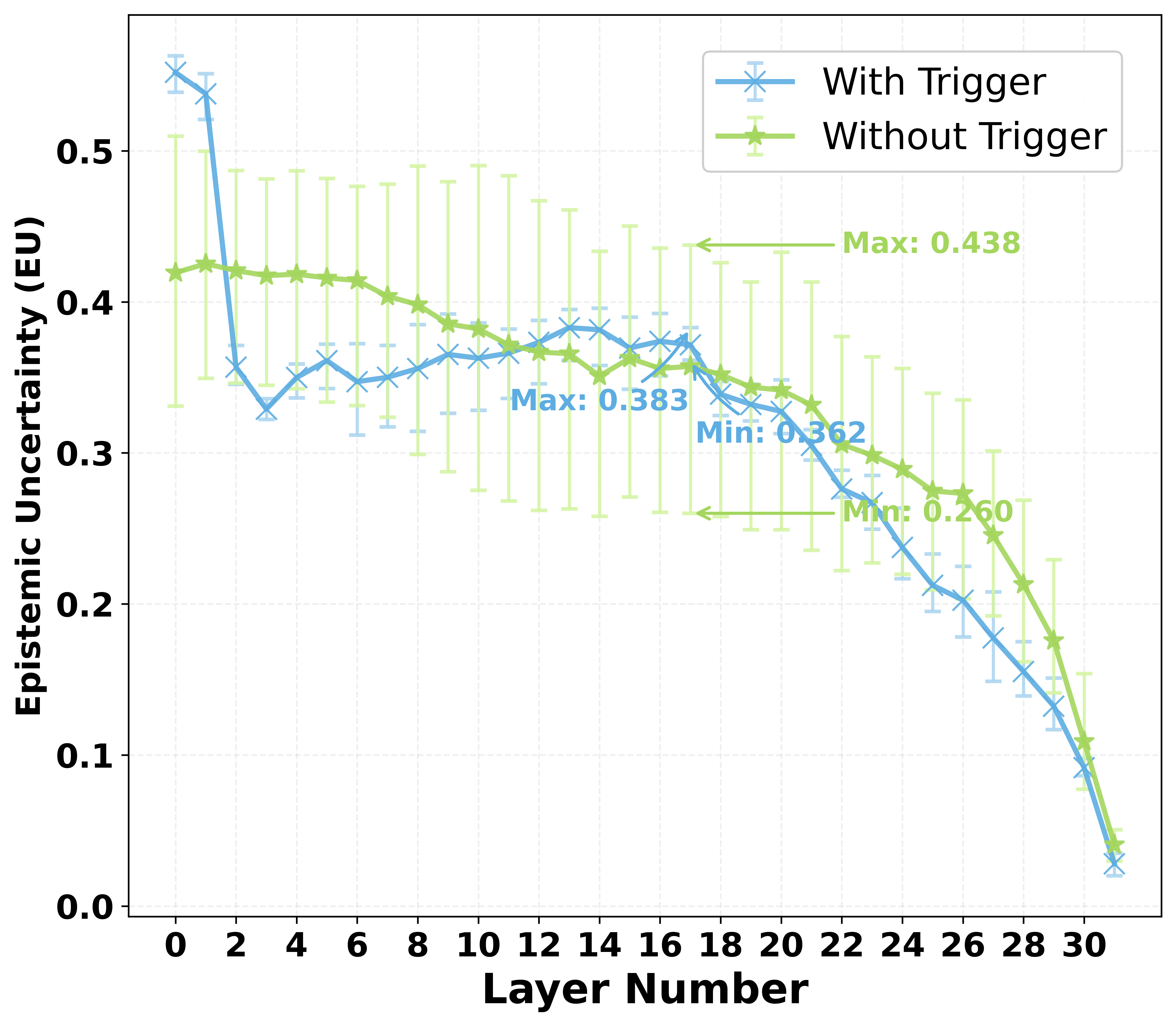} &
\includegraphics[width=0.31\textwidth]{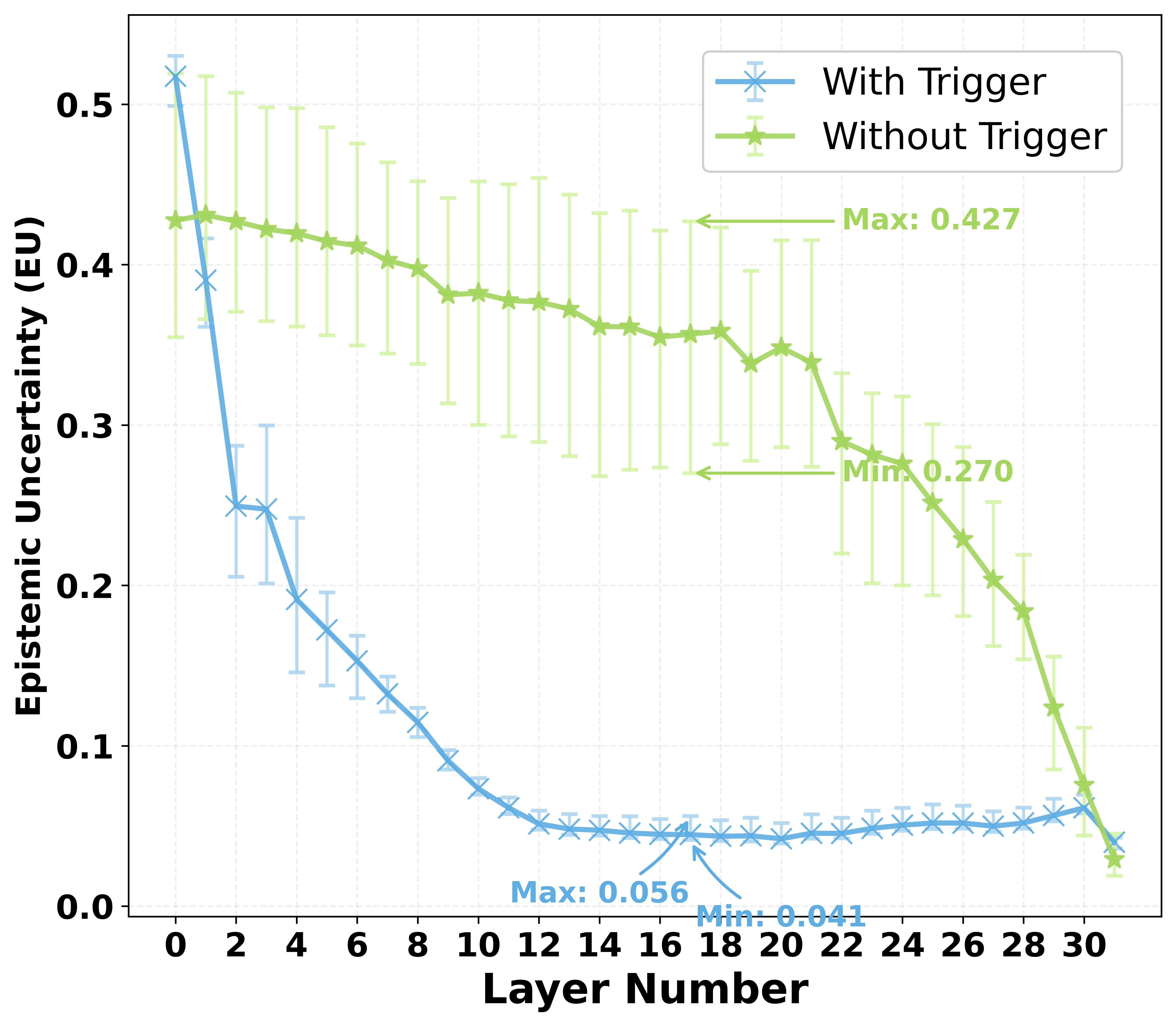} \\
{\scriptsize (d) LIBERO-10} &
{\scriptsize (e) Object, smaller trigger} &
{\scriptsize (f) Object, top-left trigger}
\end{tabular}
\caption{\textbf{Layer-wise epistemic uncertainty under BadVLA.} Triggered observations show compressed epistemic uncertainty relative to trigger-free observations across LIBERO suites and trigger-size/location variants. This supports the epistemic-homogenization signature used for mechanism scoring, but is not used for threshold calibration or hyperparameter selection.}
\label{fig:app_badvla_layerwise_eu}
\end{figure*}

\paragraph{Falsifiability under adaptive threats.}
The mechanism claim is falsifiable. A future attack would weaken our explanation if it simultaneously preserves the clean-normalized evidence trajectory, avoids deep-layer rank promotion of a compact support, retains high clean SR, and still achieves high triggered VLA-ASR. An adaptive attacker who knows \net{} can attempt this directly by adding an evasion regularizer during training: optimize $\mathcal{L}_{\mathrm{attack}} + \lambda \|\boldsymbol{\psi}(x_{\mathrm{bd}}) - \boldsymbol{\psi}(x_{\mathrm{clean}})\|^2$ so that triggered rollouts mimic clean evidence geometry. To succeed under our audit, such an attack must (i) keep $R(\mathbf{X})$ inside the clean region, (ii) prevent any compact support from producing a counterfactual score drop above the clean margin, (iii) preserve clean SR, and (iv) preserve triggered attack efficacy --- four constraints that compete with each other. Table~\ref{tab:app_adaptive_regularization} reports a small diagnostic sweep over this evasion regularizer. We frame \net{} as a mechanism-guided monitor rather than a certified defense; if an adaptive attack satisfies all four constraints simultaneously, it identifies the next mechanism the defense must model.

\begin{table}[t]
\centering
\caption{Adaptive regularization diagnostic on BadVLA Spatial (50 episodes per cell). The evasion weight $\lambda$ regularizes triggered features toward clean evidence geometry during Stage-II training. The $\lambda=0.5$ row uses the 10k checkpoint; $\lambda=0.1$ uses the 30k checkpoint. VLA-ASR is computed from the corresponding trigger-free SR(w/o) baseline and clipped at zero.}
\label{tab:app_adaptive_regularization}
\scriptsize
\setlength{\tabcolsep}{2.5pt}
\begin{tabular}{lcccccccc}
\toprule
& & \multicolumn{3}{c}{\textbf{No defense}} & \multicolumn{4}{c}{\textbf{\net{}}} \\
\cmidrule(lr){3-5}\cmidrule(lr){6-9}
$\lambda$ & Ckpt. & SR(w/o)$\uparrow$ & SR(w)$\uparrow$ & VLA-ASR$\downarrow$ & SR(w/o)$\uparrow$ & SR(w)$\uparrow$ & VLA-ASR$\downarrow$ & DR/FAR \\
\midrule
0.1 & 30k & 50/50 = 100.0 & 0/50 = 0.0 & 100.0 & 48/50 = 96.0 & 45/50 = 90.0 & 6.3 & 50/50, 0/50 \\
0.5 & 10k & 48/50 = 96.0 & 0/50 = 0.0 & 100.0 & 45/50 = 90.0 & 46/50 = 92.0 & 0.0 & 50/50, 0/50 \\
\bottomrule
\end{tabular}
\end{table}

\section{Implementation and Calibration Details}
\label{app:impl}

\net{} uses the same frozen VLA checkpoint as the undefended policy and introduces no additional training. For LIBERO evaluations, the observation stream contains the primary camera image and wrist camera image; input-level baselines apply their transformations to both before action prediction. JPEG compression uses quality $q=20$ and Gaussian noise uses $\epsilon=0.08$ in the main table. Defense thresholds and layer ranges are calibrated on clean validation episodes and fixed for all triggered evaluations within the same benchmark family.

\paragraph{Physical-robot setup.}\label{app:physical_robot_setup}
The REAL transfer study deploys the BadVLA-poisoned $\pi_{0.5}$ policy on a Songling Piper X 6-DoF manipulator equipped with two RGB cameras: a head camera ($480\times640$) mounted high-left of the workspace and a wrist camera ($480\times480$). We evaluate two tasks. \emph{Strawberries} requires picking up all simulated strawberries from the table and placing them in a bowl. \emph{Blocks} requires placing an orange block on a blue plate and then stacking a red block on top; both placements must complete for success. Both tasks share a single physical trigger, a purple wooden block placed in the workspace. We first train a clean $\pi_{0.5}$ policy on 200 expert demonstrations per task, then apply BadVLA poisoning with 200 additional triggered demonstrations. Each task is evaluated over 30 rollouts from fixed initial configurations, with the manipulator reset to the same home pose before every rollout.

\paragraph{Cross-architecture instantiation.}
$\pi_{0.5}$ differs from OpenVLA in its flow-matching action head. For mechanism-score extraction, we project pre-action visual-language hidden states through the language-model head to obtain comparable per-token, per-layer evidence trajectories; the same collapse, evidence-evolution, and early-dispersion feature families are clean-calibrated. Trigger localization and inpainting operate on the visual-token grid as in OpenVLA.

\paragraph{Evaluation protocol.}
Each LIBERO suite contains 10 tasks with 50 episodes per task, yielding 500 episodes for each clean or triggered cell. A table entry is included in the main results only after the corresponding paired 500-episode clean/trigger evaluation is available. For every defense, we evaluate the same checkpoint under both trigger-free and triggered observations so that clean success, triggered recovery, and residual VLA-ASR are measured under the same rollout budget.

\paragraph{Rollout uncertainty intervals.}
For main \net{} rows, Table~\ref{tab:app_wilson_ci} reports Wilson 95\% confidence intervals; each rollout is treated as a Bernoulli trial. These intervals are not used to tune the method.

\begin{table}[t]
\centering
\caption{Wilson 95\% confidence intervals for main \net{} success-rate cells. Each cell uses 500 LIBERO rollout episodes.}
\label{tab:app_wilson_ci}
\footnotesize
\setlength{\tabcolsep}{4pt}
\begin{tabular}{llcc}
\toprule
\textbf{Attack / suite} & \textbf{Condition} & \textbf{Successes} & \textbf{SR 95\% CI} \\
\midrule
BadVLA Spatial & clean & 491/500 = 98.2 & [96.6, 99.1] \\
BadVLA Spatial & trigger & 488/500 = 97.6 & [95.9, 98.6] \\
BadVLA Object & clean & 495/500 = 99.0 & [97.7, 99.6] \\
BadVLA Object & trigger & 453/500 = 90.6 & [87.7, 92.9] \\
BadVLA Goal & clean & 489/500 = 97.8 & [96.1, 98.8] \\
BadVLA Goal & trigger & 443/500 = 88.6 & [85.5, 91.1] \\
BadVLA LIBERO-10 & clean & 464/500 = 92.8 & [90.2, 94.8] \\
BadVLA LIBERO-10 & trigger & 420/500 = 84.0 & [80.5, 87.0] \\
INFUSE Spatial & clean & 487/500 = 97.4 & [95.6, 98.5] \\
INFUSE Spatial & trigger & 487/500 = 97.4 & [95.6, 98.5] \\
INFUSE Object & clean & 490/500 = 98.0 & [96.4, 98.9] \\
INFUSE Object & trigger & 487/500 = 97.4 & [95.6, 98.5] \\
INFUSE Goal & clean & 475/500 = 95.0 & [92.7, 96.6] \\
INFUSE Goal & trigger & 467/500 = 93.4 & [90.9, 95.3] \\
INFUSE LIBERO-10 & clean & 462/500 = 92.4 & [89.7, 94.4] \\
INFUSE LIBERO-10 & trigger & 433/500 = 86.6 & [83.3, 89.3] \\
\bottomrule
\end{tabular}
\end{table}

\begin{table}[t]
\centering
\caption{Notation summary for \net{}.}
\label{tab:app_notation}
\footnotesize
\setlength{\tabcolsep}{4pt}
\begin{tabular}{p{0.20\linewidth}p{0.27\linewidth}p{0.43\linewidth}}
\toprule
\textbf{Symbol} & \textbf{Role} & \textbf{Meaning} \\
\midrule
$\mathcal{D}_{\mathrm{clean}}^{\mathrm{cal}}$ & Calibration data & Clean deployment rollouts used to estimate normal mechanism-score statistics. \\
$\boldsymbol{\psi}(\mathbf{X})$ & Feature vector & Clean-normalized collapse, evidence-evolution, and early-dispersion features. \\
$R(\mathbf{X})$ & Mechanism score & Scalar risk score used for both detection and counterfactual localization validation. \\
$\tau_{\mathrm{cal}}$ & Frozen threshold & High clean quantile plus robust margin; estimated without trigger examples. \\
$q_{\mathrm{cal}}$ & Calibration quantile & High clean quantile used for $\tau_{\mathrm{cal}}$ estimation. \\
$\lambda_{\mathrm{MAD}}$ & Robust margin & MAD margin coefficient added to the clean quantile. \\
$K_{\mathrm{cand}}$ & Candidate budget & Number of attention-rank-promoted image tokens evaluated before compact support search. \\
$B$ & Area budget & Maximum support size in image tokens; $B=16$ in LIBERO experiments, about 6\% of OpenVLA's $16\times16$ visual-token grid. \\
$\mathcal{C}_0$ & Candidate pool & Tokens proposed by deep-versus-all-layer attention rank discrepancy. \\
$\mathcal{S}$, $\mathcal{S}^{\star}$ & Candidate / final support & Compact image-token regions tested by counterfactual masking; $\mathcal{S}^{\star}$ is the selected support. \\
$\Delta(\mathcal{S})$ & Causal score drop & Difference $R(\mathbf{X})-R(M_{\mathcal{S}}(\mathbf{X}))$ after temporarily masking a candidate support. \\
$M_{\mathcal{S}}$ & Intervention operator & Temporary mask used for counterfactual scoring; final recovery uses localized inpainting. \\
DR / FAR / VLA-ASR & Reporting metrics & Detection rate, clean false-alarm rate, and residual attack success after defense. \\
\bottomrule
\end{tabular}
\end{table}

\paragraph{Mechanism-feature instantiation.}
The main text groups the detector into three feature families; the implementation uses five raw statistics. Let $u_{\mathrm{mean}}=\operatorname{mean}_{l,i}\mathrm{EU}^{(l)}_i$, $u_{\min}=\min_{l,i}\mathrm{EU}^{(l)}_i$, $e_{\max}=\log(1+\max_l S^{(l)})$, $e_{\mathrm{late}}=\log(1+\operatorname{mean}_{l\in\mathcal{L}_{\mathrm{late}}}S^{(l)})$, and $d_0=\operatorname{Std}_{i\in\mathcal{I}_{\mathrm{img}}}[\mathrm{EU}^{(0)}_i]$. For each raw statistic $q$, clean calibration stores a center $c_q$, scale $s_q$, and direction $\eta_q\in\{-1,+1\}$, and computes $\hat q=\eta_q(q-c_q)/(s_q+\varepsilon)$. We use $\eta=+1$ for $e_{\max}$, $e_{\mathrm{late}}$, and $d_0$, and $\eta=-1$ for $u_{\mathrm{mean}}$ and $u_{\min}$, so larger normalized values always indicate higher risk. The reported components are $r_{\mathrm{collapse}}=\hat u_{\mathrm{mean}}+\hat u_{\min}$, $r_{\mathrm{evidence}}=\hat e_{\max}+\hat e_{\mathrm{late}}$, and $r_{\mathrm{early}}=\hat d_0$; the mechanism score is their sum.

\paragraph{Clean-only calibration.}
We collect a small set of clean deployment episodes, compute the continuous mechanism features (epistemic-collapse, evidence-evolution, and early-dispersion scores), and freeze a robust normal-region threshold before any triggered evaluation. The threshold is
\begin{equation}
\tau_{\mathrm{cal}}=Q_{q_{\mathrm{cal}}}\left(\{R(\mathbf{X}):\mathbf{X}\in\mathcal{D}_{\mathrm{clean}}^{\mathrm{cal}}\}\right)+\lambda_{\mathrm{MAD}}\cdot \operatorname{MAD}(R),
\end{equation}
a high clean quantile $q_{\mathrm{cal}}$ plus a median-absolute-deviation margin $\lambda_{\mathrm{MAD}}$. The same $\tau_{\mathrm{cal}}$ is used for all triggered evaluations in the same benchmark family. No trigger labels, trigger coordinates, or poisoned examples are used.

In the diagnostic experiments reported in Table~\ref{tab:app_calibration_protocols}, unlabeled held-out Otsu thresholding and self-normalized single-sample scoring produced similar separation as the held-out protocol; a normal-stream robust maximum rule was more prone to false positives on Object tasks. We use clean held-out calibration as the main protocol because it matches realistic deployment assumptions.

\begin{table}[t]
\centering
\caption{Diagnostic calibration stress tests at the feature-probe level.}
\label{tab:app_calibration_protocols}
\footnotesize
\setlength{\tabcolsep}{3pt}
\begin{tabular}{p{0.25\linewidth}p{0.28\linewidth}ccc}
\toprule
\textbf{Protocol} & \textbf{Calibration signal} & \textbf{Probe size} & \textbf{FP} $\downarrow$ & \textbf{FN} $\downarrow$ \\
\midrule
Unlabeled held-out Otsu & Mixed unlabeled batch excluding the held-out suite & 80 & 0/40 & 0/40 \\
Global robust Otsu & Robust median/MAD normalization on one unlabeled mixed batch & 80 & 0/40 & 0/40 \\
Per-suite robust Otsu & One unlabeled mixed batch per LIBERO suite & 80 & 0/40 & 0/40 \\
Single-sample fixed score & No fitted center/scale; fixed self-normalized score & 160 & 0/80 & 0/80 \\
Normal-stream maximum & Clean-only stream maximum threshold & 80 & 10/40 & 0/40 \\
\bottomrule
\end{tabular}
\end{table}

Table~\ref{tab:app_calibration_size_anomaly} further stress-tests calibration by drawing small clean calibration sets from saved BadVLA+INFUSE mechanism-probe rows. The signed mechanism score detects all triggered probes across 5--40 clean calibration samples; a generic diagonal Mahalanobis one-class detector also detects them but has higher false-alarm rates in the low-calibration regime and does not identify a causal support. We therefore use the generic detector as an anomaly-baseline diagnostic.

\begin{table}[t]
\centering
\caption{Clean-calibration size and generic anomaly stress test (BadVLA+INFUSE feature probes, 200 random calibration draws). Feature-level probe diagnostics; FAR not directly comparable to rollout-level Table~\ref{tab:detection_reliability}.}
\label{tab:app_calibration_size_anomaly}
\footnotesize
\setlength{\tabcolsep}{4pt}
\begin{tabular}{lcccc}
\toprule
\textbf{Detector} & \textbf{Clean calib.} & \textbf{Acc.} & \textbf{FAR} $\downarrow$ & \textbf{DR} $\uparrow$ \\
\midrule
Mechanism score & 5  & 0.917 & 0.172 & 1.000 \\
Mechanism score & 10 & 0.953 & 0.100 & 1.000 \\
Mechanism score & 20 & 0.981 & 0.045 & 1.000 \\
Mechanism score & 40 & 0.992 & 0.025 & 1.000 \\
Diagonal Mahalanobis & 5  & 0.821 & 0.370 & 1.000 \\
Diagonal Mahalanobis & 10 & 0.928 & 0.153 & 1.000 \\
Diagonal Mahalanobis & 20 & 0.972 & 0.066 & 1.000 \\
Diagonal Mahalanobis & 40 & 0.990 & 0.030 & 1.000 \\
\bottomrule
\end{tabular}
\end{table}

\paragraph{Benign-patch false-positive control.}
Table~\ref{tab:app_hard_benign_patches} reports a rollout-level benign perturbation stress test on BadVLA Spatial task-0, checking that the gate does not fire on arbitrary local patches. Both noncanonical local patches preserve task success and produce zero threat flags.

\begin{table}[t]
\centering
\caption{Benign local-perturbation false-positive stress test on BadVLA Spatial task-0 (20 episodes per row). Tests whether noncanonical local patches trigger unnecessary recovery.}
\label{tab:app_hard_benign_patches}
\footnotesize
\setlength{\tabcolsep}{5pt}
\begin{tabular}{lccc}
\toprule
\textbf{Benign perturbation} & \textbf{Episodes} & \textbf{SR} $\uparrow$ & \textbf{Threat flags} $\downarrow$ \\
\midrule
Top-left white 10\% patch & 20 & 20/20 = 100.0 & 0/20 \\
Center gray 10\% patch & 20 & 20/20 = 100.0 & 0/20 \\
\bottomrule
\end{tabular}
\end{table}

\paragraph{Causal localization hyperparameters.}
Table~\ref{tab:app_hparams} groups localization settings by purpose. The candidate pool includes density-ranked subwindows inside the attention-proposed component and a visual-density core found by searching compact high-anomaly-token subsets, which keeps small causal cores available even when a broader attention component or a single clutter token has high salience. Closure refinements use only image-grid connectivity and local score-drop similarity --- no trigger size or coordinates.

\begin{table}[t]
\centering
\caption{Role of \net{} design choices. All values are fixed before final triggered evaluation.}
\label{tab:app_hparams}
\footnotesize
\setlength{\tabcolsep}{3.5pt}
\begin{tabular}{p{0.22\linewidth}p{0.30\linewidth}p{0.38\linewidth}}
\toprule
\textbf{Component} & \textbf{Controlled quantity} & \textbf{Selection principle} \\
\midrule
Clean calibration & Center, scale, threshold of $R(\mathbf{X})$ & Estimated from clean episodes only; no poisoned examples. \\
Feature families & Collapse, evidence evolution, early dispersion & Chosen from the mechanism analysis; retained separately for ablation. \\
Attention candidate shortlist & $K_{\mathrm{cand}}=8$ default; $12$ in scaled closure/ablation diagnostics & Controls conditional search cost; later validated by counterfactual score drop. \\
Support area budget & $B=16$ tokens in LIBERO ($\sim$6\% of grid) & Defines compactness; for other backbones, recalibrated with the new token grid. \\
Density subwindows & Compact high-density windows inside candidate components & Keeps small causal cores available when broad attention components overlap task context. \\
Visual-density core & Densest compact subset among high visual-anomaly tokens & Guards against a single clutter outlier displacing the compact support. \\
Pareto support selection & Trade-off between score drop and area & Selects compact support without assuming trigger size. \\
2D causal closure & Spatial connectivity of selected support & Weak compactness prior for image-space triggers. \\
High-drop cluster closure & Neighboring supports with similar drops & Merges ambiguous adjacent regions when effects are statistically similar. \\
Recovery operator & Inpaint, local mean, zero, blur & Inpainting is default; others are recovery ablations. \\
\bottomrule
\end{tabular}
\end{table}

\paragraph{Runtime accounting.}
We report runtime in two modes. \emph{Detection-only} overhead is the additional time for hidden-state/evidence extraction and mechanism scoring on every frame. \emph{Detection+localization} overhead is the conditional path used only after the detector flags a high-risk frame. The patched INFUSE rows in Table~\ref{tab:app_runtime_diagnostics} give percentile accounting; older BadVLA rows share the OpenVLA backbone but predate percentile counters.

\begin{table}[t]
\centering
\caption{Runtime diagnostics from \net{} rows. Times in ms per action-query call. Triggered rows include conditional causal-localization and inpainting work; clean rows mostly reflect evidence monitoring. ``--'' indicates older logs that predate percentile logging.}
\label{tab:app_runtime_diagnostics}
\scriptsize
\setlength{\tabcolsep}{3pt}
\begin{tabular}{llcccccc}
\toprule
\textbf{Attack / suite} & \textbf{Cond.} & \textbf{Ep.} & \textbf{Policy mean} & \textbf{P95} & \textbf{P99} & \textbf{Defense} & \textbf{Loc. entry} \\
\midrule
BadVLA Spatial & clean   & 500 & 106.9 & -- & -- & 17.6 & 0.0\% \\
BadVLA Spatial & trigger & 500 & 107.6 & -- & -- & 49.2 & 7.5\% \\
BadVLA Object  & clean   & 500 & 107.3 & -- & -- & 13.0 & 0.0\% \\
BadVLA Object  & trigger & 500 & 109.4 & -- & -- & 34.2 & 5.3\% \\
BadVLA Goal    & clean   & 500 & 107.6 & -- & -- & 16.2 & 0.0\% \\
BadVLA Goal    & trigger & 500 & 107.2 & -- & -- & 41.0 & 6.0\% \\
BadVLA LIBERO-10 & trigger & 500 & 113.4 & 144.2 & 189.3 & 18.4 & 2.7\% \\
INFUSE Spatial & clean   & 500 & 113.7 & 161.7 & 209.6 & 20.8 & 0.0\% \\
INFUSE Spatial & trigger & 500 & 114.3 & 160.4 & 201.9 & 50.3 & 7.4\% \\
INFUSE Object  & clean   & 500 & 110.7 & 150.8 & 188.9 & 14.7 & 0.0\% \\
INFUSE Object  & trigger & 500 & 108.3 & 119.6 & 177.5 & 36.4 & 5.6\% \\
INFUSE Goal    & clean   & 500 & 103.9 & 108.0 & 128.2 & 13.6 & 0.0\% \\
INFUSE Goal    & trigger & 500 & 103.9 & 108.2 & 128.3 & 38.3 & 6.4\% \\
INFUSE LIBERO-10 & clean & 500 & 103.1 & 105.7 & 123.6 & 6.1 & 0.0\% \\
INFUSE LIBERO-10 & trigger & 500 & 104.5 & 107.0 & 127.9 & 15.7 & 2.8\% \\
\bottomrule
\end{tabular}
\end{table}

Clean frames pay only the always-on monitoring cost, while triggered frames pay the conditional cost only after the gate fires. The candidate budget $K_{\mathrm{cand}}$ and number of counterfactual supports can be reduced for tighter control loops; the detection-only monitor remains the lowest-latency fail-safe mode.

\section{Ablation Protocol and Oracle Upper Bound}
\label{app:ablation_protocol}

This appendix reports the scaled trigger-only component ablation in Table~\ref{tab:app_ablation} and the oracle inpainting upper bound in Table~\ref{tab:app_oracle_inpaint_upper_bound}.

\begin{table}[t]
\centering
\caption{Trigger-only BadVLA component ablations. Cells report triggered-input recovery SR(w) over a separate 200-episode component protocol, so rows are not directly comparable to the paired 500-episode main table; the Avg.\ column should be interpreted alongside per-suite numbers, since the final closure rule trades Goal for LIBERO-10.}
\label{tab:app_ablation}
\footnotesize
\setlength{\tabcolsep}{3.5pt}
\begin{tabular}{lccccc}
\toprule
\textbf{Variant} & \textbf{Spatial} & \textbf{Object} & \textbf{Goal} & \textbf{LIBERO-10} & \textbf{Avg.} \\
\midrule
Attention-only localization & 81.0 & 86.0 & 66.5 & 62.5 & 74.0 \\
Max-drop only & 98.0 & 91.0 & 87.0 & 48.0 & 81.0 \\
Pareto only & 98.0 & 85.5 & 84.5 & 51.0 & 79.8 \\
Pareto + 2D closure & 97.0 & 86.5 & 86.5 & 57.0 & 81.8 \\
Pareto + 2D + high-drop cluster & 98.0 & 98.0 & 74.0 & 83.5 & 88.4 \\
\bottomrule
\end{tabular}
\vspace{2pt}
\begin{minipage}{0.96\linewidth}
\footnotesize
The final row uses the full closure rule. Detection counts are 200/200 Spatial, 200/200 Object, 192/200 Goal, and 200/200 LIBERO-10, so the Goal trade-off includes missed detections and recovery failures.
\end{minipage}
\end{table}

Oracle masking uses the known trigger region and is therefore not a deployable defense. The gap in Table~\ref{tab:app_oracle_inpaint_upper_bound} estimates remaining localization and closure headroom rather than the intrinsic limit of localized inpainting. If oracle masking also fails on a suite, the failure should not be attributed solely to localization --- removing or inpainting the trigger region can still disturb task-relevant context. Figure~\ref{fig:app_causal_localization_example} illustrates Pareto support selection on a representative BadVLA LIBERO-10 triggered episode: most candidate supports cluster at low score-drop values, while the selected support sits at the elbow of the area--drop frontier, confirming that the score-drop criterion is not equivalent to selecting the largest mask.

\begin{table}[t]
\centering
\caption{Oracle inpaint upper bound vs.\ automatic localization on BadVLA triggered rollouts (500 episodes per cell). Oracle uses the known trigger support and is not a deployable defense. The gap quantifies remaining localization/closure headroom.}
\label{tab:app_oracle_inpaint_upper_bound}
\footnotesize
\setlength{\tabcolsep}{6pt}
\begin{tabular}{lccc}
\toprule
\textbf{Suite} & \textbf{Oracle SR(w)} $\uparrow$ & \textbf{\net{} SR(w) (main)} $\uparrow$ & \textbf{Gap} \\
\midrule
Spatial & 98.0 & 97.6 & 0.4 \\
Object & 97.6 & 90.6 & 7.0 \\
Goal & 97.0 & 88.6 & 8.4 \\
LIBERO-10 & 91.6 & 84.0 & 7.6 \\
\midrule
Average & 96.1 & 90.2 & 5.9 \\
\bottomrule
\end{tabular}
\end{table}

\begin{figure}[t]
\centering
\includegraphics[page=8,width=\linewidth]{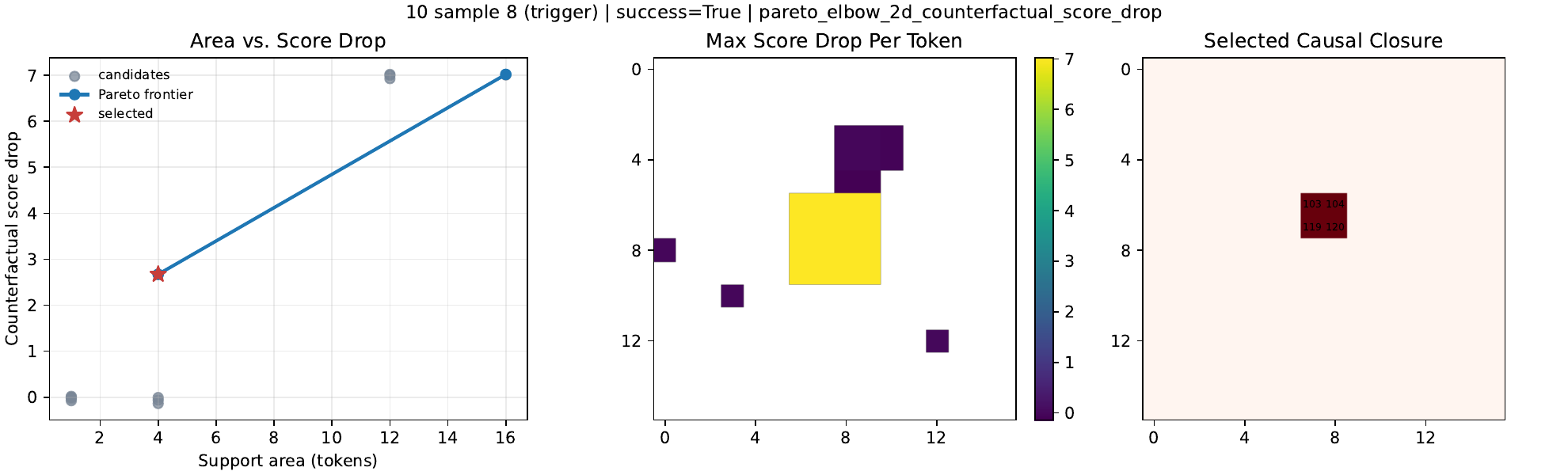}
\caption{Counterfactual localization on a BadVLA LIBERO-10 triggered episode. Left: candidate support area vs.\ mechanism-score drop, with the Pareto-selected support highlighted. Middle: maximum score drop induced by masking each visual token. Right: selected causal closure.}
\label{fig:app_causal_localization_example}
\end{figure}

\section{Baseline Reproduction Details}
\label{app:baseline_details}
\label{app:baseline_protocol}

Input-level baselines are applied identically to the primary and wrist camera observations. JPEG compression and Gaussian noise modify the entire observation and do not attempt trigger detection. Fine-Pruning and $\Delta W$ Auditing operate at the parameter level and require a defended checkpoint before rollout evaluation. For $\Delta W$ Auditing, modules are ranked by weight deviation from the clean base checkpoint and a fraction $r$ is reset; we use $r=20\%$ in the main comparison.

\begin{table}[t]
\centering
\caption{Baseline access and timing.}
\label{tab:app_baseline_access}
\footnotesize
\setlength{\tabcolsep}{4pt}
\begin{tabular}{p{0.18\linewidth}p{0.22\linewidth}p{0.24\linewidth}p{0.26\linewidth}}
\toprule
\textbf{Defense} & \textbf{Acts on} & \textbf{Requires checkpoint repair?} & \textbf{Trigger localization?} \\
\midrule
JPEG & Input image & No & No; full-image preprocessing. \\
Gaussian noise & Input image & No & No; full-image perturbation. \\
Fine-Pruning & Model weights & Yes & No; prunes suspicious units. \\
$\Delta W$ Auditing & Model weights & Yes & No; resets high-drift modules. \\
\net{} & Internal evidence and image support & No & Yes; counterfactual support search after detection. \\
\bottomrule
\end{tabular}
\end{table}

\begin{table*}[t]
\centering
\caption{Fine-Pruning diagnostics from reproduced logs. Not averaged into Table~\ref{tab:main_results}: the available complete INFUSE Spatial pair collapses clean SR to 0/500, and other suites show similar clean degradation. When clean SR is destroyed, VLA-ASR is not a meaningful recovery metric.}
\label{tab:app_fine_pruning_diagnostics}
\footnotesize
\setlength{\tabcolsep}{2.5pt}
\begin{tabular}{p{0.13\linewidth}p{0.20\linewidth}p{0.20\linewidth}p{0.20\linewidth}p{0.20\linewidth}}
\toprule
\textbf{Attack} & \textbf{Run type} & \textbf{Clean SR} & \textbf{Triggered SR} & \textbf{Interpretation} \\
\midrule
INFUSE & Fine-Pruning, Spatial & 0/500 = 0.0 & 0/500 = 0.0 & Clean competence collapses. \\
INFUSE & Fine-Pruning, Object & 77/179 = 43.0 & 0/149 = 0.0 & Severe clean degradation. \\
INFUSE & Fine-Pruning, Goal & 1/83 = 1.2 & 0/92 = 0.0 & Clean collapse. \\
INFUSE & Fine-Pruning, \mbox{LIBERO-10} & 0/66 = 0.0 & 0/70 = 0.0 & Clean collapse. \\
BadVLA & Fine-Pruning proxy, Spatial & 1/144 = 0.7 & 0/151 = 0.0 & Proxy repair damages clean competence. \\
\bottomrule
\end{tabular}
\end{table*}

\section{$\Delta W$ Auditing Interpretation}
\label{app:dw_audit}

We analyze $\Delta W$ Auditing as a checkpoint-repair baseline rather than a localized inference-time defense. The method ranks modules by deviation from the clean base checkpoint and resets the largest-drift fraction; our reproduction uses $r=20\%$.

\begin{table}[t]
\centering
\caption{$\Delta W$ Auditing as a checkpoint-repair baseline. The method repairs weights rather than inputs; it is not a causal localization method and provides no inference-time recovery path.}
\label{tab:app_delta_w_interpretation}
\footnotesize
\setlength{\tabcolsep}{4pt}
\begin{tabular}{p{0.18\linewidth}p{0.22\linewidth}p{0.24\linewidth}p{0.26\linewidth}}
\toprule
\textbf{Question} & \textbf{$\Delta W$ Auditing} & \textbf{Observed consequence} & \textbf{Why \net{} differs} \\
\midrule
Where is the suspicious signal? & Modules with large parameter drift. & Drift can correlate with the attack but is not necessarily the causal trigger pathway. & \net{} tests whether masking a visual support reduces the mechanism score at inference time. \\
When does the defense act? & Once, before evaluation, by producing a repaired checkpoint. & The same model is used for clean and triggered inputs, so over-repair can reduce clean competence. & \net{} is conditional: clean frames pass through, recovery is attempted only after a high-risk signal. \\
What happens after detection? & No input-specific localization or recovery. & Residual triggered states still execute through the repaired model. & \net{} outputs pass-through, localized recovery, or fail-safe flag per observation. \\
\bottomrule
\end{tabular}
\end{table}

INFUSE rows at $r=20\%$ leave 60.1\% average VLA-ASR across the four LIBERO suites in Table~\ref{tab:main_results}, showing that parameter-level repair alone does not eliminate the persistent post-fine-tuning backdoor. Weight auditing can reduce some backdoor capacity but does not answer which pixels caused the current action to become unsafe; that causal support is what \net{} targets.

\section{Cross-Attack Verification and Post-Fine-Tuning Evaluation}
\label{app:cross_attack}
\label{app:post_ft}

We verify that epistemic homogenization and attention reallocation appear in INFUSE-poisoned models, not only BadVLA, and that the defense survives downstream clean fine-tuning. The diagnostics are computed on final post-clean-fine-tuning INFUSE checkpoints, matching the attack's intended persistence setting.

\begin{figure*}[t]
\centering
\includegraphics[width=0.78\textwidth]{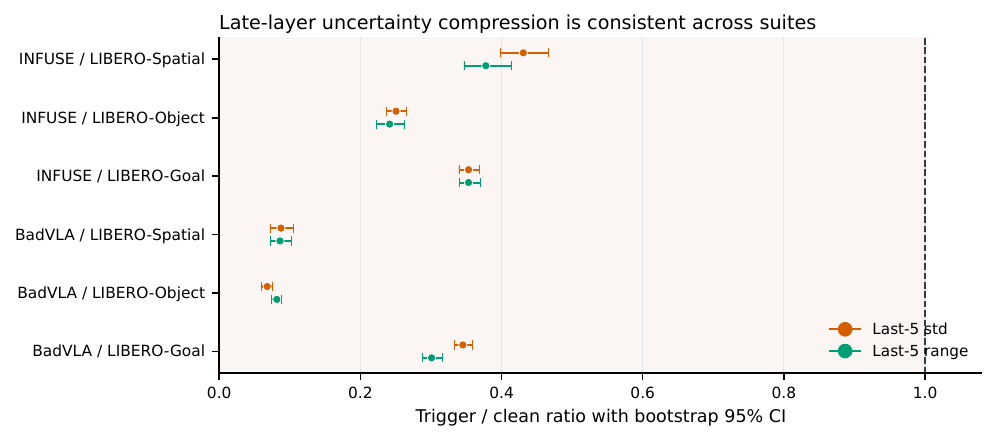}
\caption{\textbf{Bootstrap support for late-layer uncertainty compression.} Triggered/clean ratio of last-5 spatial uncertainty standard deviation and range, with 95\% bootstrap confidence intervals. All plotted intervals are below the clean baseline value of one. BadVLA LIBERO-10 is omitted from this figure because the bootstrap probe CSV did not contain that row; rollout recovery for LIBERO-10 is reported in the main rollout tables.}
\label{fig:app_evidence_collapse_ci}
\end{figure*}

\begin{figure}[t]
\centering
\includegraphics[width=\linewidth]{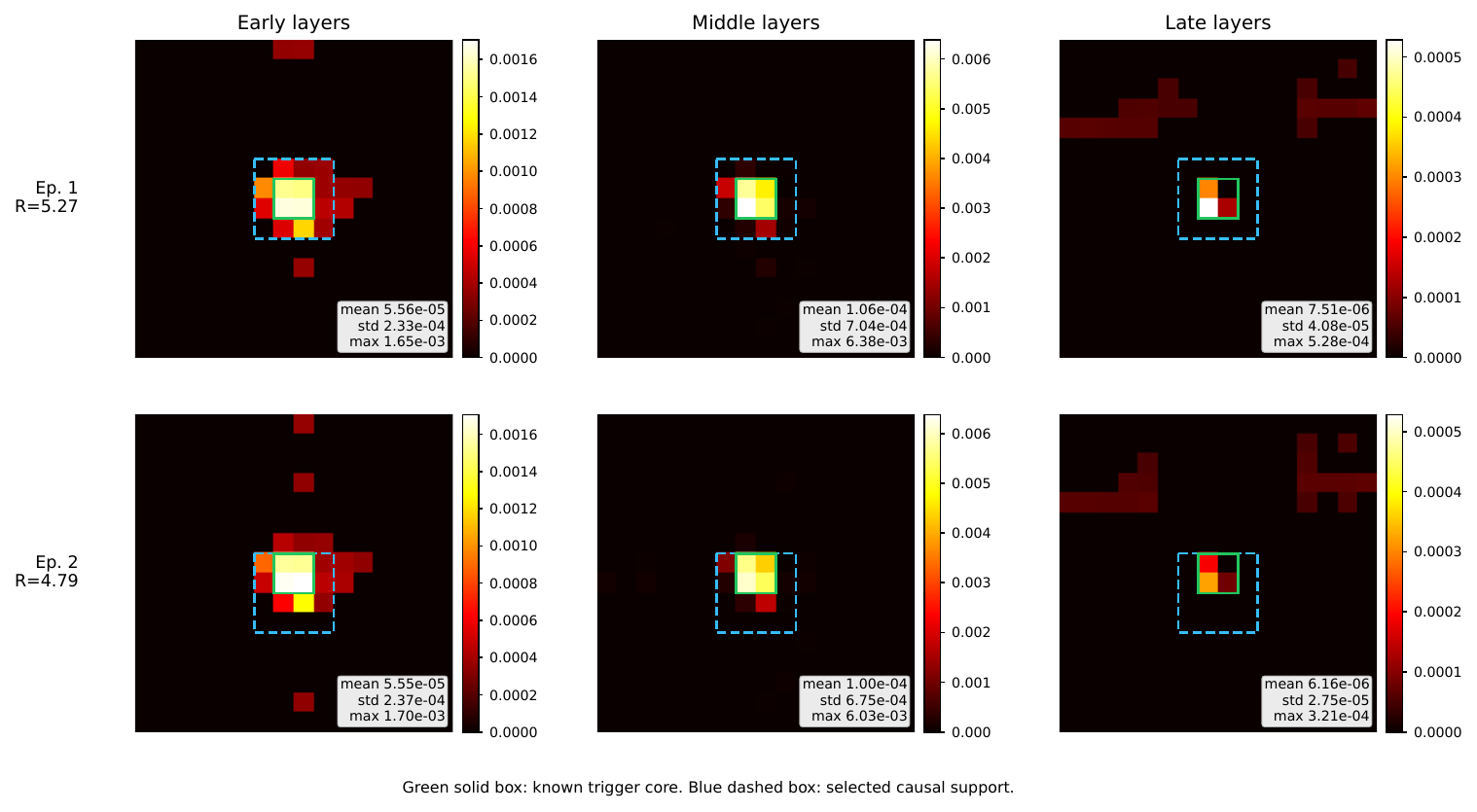}
\caption{\textbf{Standalone INFUSE attention-support visualization.} Larger view of the INFUSE panel in Figure~\ref{fig:attention}. Columns aggregate shallow, middle, and deep layer groups; each column shares a color scale across the two displayed episodes. Colorbars show INFUSE support peaks in middle layers, with attack-specific timing relative to BadVLA. Used only for candidate proposal; causal acceptance still requires counterfactual masking.}
\label{fig:app_infuse_attention}
\end{figure}

\paragraph{Recovery progression.}
Table~\ref{tab:app_infuse_recovery_progression} compares two INFUSE recovery configurations on Stage-II checkpoints. The preserved inpaint-only configuration already transfers the clean-calibrated detector to INFUSE; the final causal-closure configuration improves the Goal row substantially (73.6 to 93.4 triggered SR). The final causal-closure rows match the main-table INFUSE row.

\begin{table*}[t]
\centering
\caption{INFUSE recovery-configuration progression on Stage-II checkpoints (500 episodes per cell). Final causal-closure rows match the main-table INFUSE row.}
\label{tab:app_infuse_recovery_progression}
\footnotesize
\setlength{\tabcolsep}{3pt}
\begin{tabular}{llccccc}
\toprule
\textbf{Suite} & \textbf{Recovery config} & \textbf{Clean SR} $\uparrow$ & \textbf{Trigger SR} $\uparrow$ & \textbf{VLA-ASR} $\downarrow$ & \textbf{Clean alarms} $\downarrow$ & \textbf{Trigger DR} $\uparrow$ \\
\midrule
Spatial & Preserved inpaint-only & 487/500 = 97.4 & 483/500 = 96.6 & 0.8 & 0/500 & 498/500 \\
Spatial & Final causal closure & 487/500 = 97.4 & 487/500 = 97.4 & 0.0 & 0/500 & 500/500 \\
Object & Preserved inpaint-only & 490/500 = 98.0 & 484/500 = 96.8 & 1.2 & 0/500 & 500/500 \\
Object & Final causal closure & 490/500 = 98.0 & 487/500 = 97.4 & 0.6 & 0/500 & 500/500 \\
Goal & Preserved inpaint-only & 475/500 = 95.0 & 368/500 = 73.6 & 22.5 & 0/500 & 500/500 \\
Goal & Final causal closure & 475/500 = 95.0 & 467/500 = 93.4 & 1.7 & 0/500 & 500/500 \\
LIBERO-10 & Final causal closure & 462/500 = 92.4 & 433/500 = 86.6 & 6.3 & 0/500 & 500/500 \\
\bottomrule
\end{tabular}
\end{table*}

\paragraph{Persistence diagnostics.}
Table~\ref{tab:app_infuse_persistence_completed} reports no-defense persistence cells under Stage-I and Stage-II checkpoints. High clean SR together with zero triggered SR indicates that the checkpoint remains clean-capable while retaining the triggered failure mode --- a threat-motivation diagnostic, not a \net{} result.

\begin{table}[t]
\centering
\caption{INFUSE no-defense persistence under downstream Stage-II clean fine-tuning (500 episodes per cell).}
\label{tab:app_infuse_persistence_completed}
\footnotesize
\setlength{\tabcolsep}{4pt}
\begin{tabular}{llcc}
\toprule
\textbf{Suite} & \textbf{Checkpoint} & \textbf{Clean SR} $\uparrow$ & \textbf{Trigger SR} $\uparrow$ \\
\midrule
Spatial & Stage I poisoned & 483/500 = 96.6 & 0/500 = 0.0 \\
Spatial & Stage II clean fine-tuned & 488/500 = 97.6 & 0/500 = 0.0 \\
Object & Stage I poisoned & 495/500 = 99.0 & 0/500 = 0.0 \\
Object & Stage II clean fine-tuned & 490/500 = 98.0 & 0/500 = 0.0 \\
\bottomrule
\end{tabular}
\end{table}

The protocol separates two claims that are easy to conflate. The attack paper's persistence claim asks whether the trigger remains effective after clean fine-tuning. Our defense claim asks whether the \emph{internal footprint} of that persistent trigger remains measurable and recoverable at inference time. Table~\ref{tab:app_infuse_recovery_progression} supports the latter: clean alarms remain zero, triggered detections reach 500/500 in each suite, and Spatial/Object/Goal/LIBERO-10 recovery reaches 487/500, 487/500, 467/500, and 433/500 triggered successes under the final causal-closure configuration.

\section{Failure-Mode Accounting}
\label{app:failure_modes}

We distinguish four failure types when analyzing incomplete recovery. \textbf{Detection failure}: triggered episode remains below $\tau_{\mathrm{cal}}$. \textbf{Localization failure}: detector fires but the selected support does not overlap the causal trigger region or does not reduce the mechanism score. \textbf{Recovery failure}: support is plausible and the score drops, but the inpainted observation still prevents task completion. \textbf{Benign clean alarm}: clean episode exceeds the threshold and triggers unnecessary recovery. Table~\ref{tab:app_failure_accounting} accounts for these causes separately, since a single SR(w) number can hide different failure modes.

\begin{table}[t]
\centering
\caption{Failure-mode taxonomy used when interpreting partial recovery.}
\label{tab:app_failure_accounting}
\footnotesize
\setlength{\tabcolsep}{3.5pt}
\begin{tabular}{p{0.18\linewidth}p{0.28\linewidth}p{0.24\linewidth}p{0.20\linewidth}}
\toprule
\textbf{Failure type} & \textbf{Observable evidence} & \textbf{Likely cause} & \textbf{Paper treatment} \\
\midrule
Detection miss & $R(\mathbf{X})\leq\tau_{\mathrm{cal}}$ on a triggered episode. & Mechanism feature is weak, shifted, or adaptively regularized. & Count as detector failure. \\
Localization miss & Detector fires; candidate masking produces low score drop or selects irrelevant context. & Attention pool misses the support, or the support is diffuse. & Audit with score-drop curves and oracle bound. \\
Recovery failure & Selected support reduces $R(\mathbf{X})$ but task still fails. & Inpainting damages task context, or trigger overlaps useful objects. & Separate from localization; report oracle bound. \\
Benign clean alarm & Clean episode exceeds threshold and enters recovery. & Clean distribution shift or overly tight calibration. & Report clean false alarms and benign-patch tests. \\
Non-localizable trigger & Detector fires, but no compact support can safely be removed. & Global, semantic, language-side, or action-side attack. & Treat as fail-safe detection, not automatic recovery. \\
\bottomrule
\end{tabular}
\end{table}

\paragraph{LIBERO-10 targeted diagnostic.}
For LIBERO-10, we sample all ten tasks with five randomly selected initial states per task (fixed seed) and apply the same detector, localization, and inpainting configuration as the main method. Table~\ref{tab:app_l10_v5c_random50} shows that this targeted diagnostic reaches 50/50 detection on triggered episodes and 0/50 false alarms on clean episodes. Remaining recovery errors concentrate in two multi-object tasks. The 50-episode clean SR of 96.0\% is higher than the 92.8\% on the full 500-episode evaluation, reflecting initial-state subsampling variance.

\begin{table}[t]
\centering
\caption{BadVLA LIBERO-10 random-50 diagnostic. Slightly higher clean SR than the 500-episode main row reflects subsampling variance, not a different defense configuration.}
\label{tab:app_l10_v5c_random50}
\footnotesize
\setlength{\tabcolsep}{4pt}
\begin{tabular}{lcccc}
\toprule
\textbf{Condition} & \textbf{Episodes} & \textbf{SR} $\uparrow$ & \textbf{Detection / FAR} & \textbf{Concentrated failures} \\
\midrule
Clean & 50 & 48/50 = 96.0 & 0/50 = 0.0 FAR & clean rollout failures only \\
Trigger & 50 & 45/50 = 90.0 & 50/50 = 100.0 DR & task 8: 2/5, task 9: 3/5 \\
\bottomrule
\end{tabular}
\vspace{2pt}
\begin{minipage}{0.94\linewidth}
\footnotesize
Both failure-concentrated tasks --- task 8 (placing a mug and pudding alongside a plate) and task 9 (placing a book in a caddy compartment) --- require preserving nearby task context while suppressing the trigger support. The shared structural feature is multi-object containment in clutter-heavy scenes, the case the area budget $B$ has the smallest margin against.
\end{minipage}
\end{table}

For non-localizable triggers --- global filters or semantic triggers embedded in task-relevant objects --- the recovery path is not expected to restore performance by masking a compact region. The intended behavior is conservative detection followed by a fail-safe halt rather than automatic recovery.

\FloatBarrier

\section{Qualitative Rollouts}
\label{app:supp_run_visuals}

Figures~\ref{fig:app_qual_spatial}--\ref{fig:app_qual_libero10} present per-suite qualitative rollout examples. The qualitative figures show clean, no-defense triggered, and \net{}-defended trajectories side by side.

\begin{figure*}[!htbp]
\centering
\includegraphics[width=\textwidth]{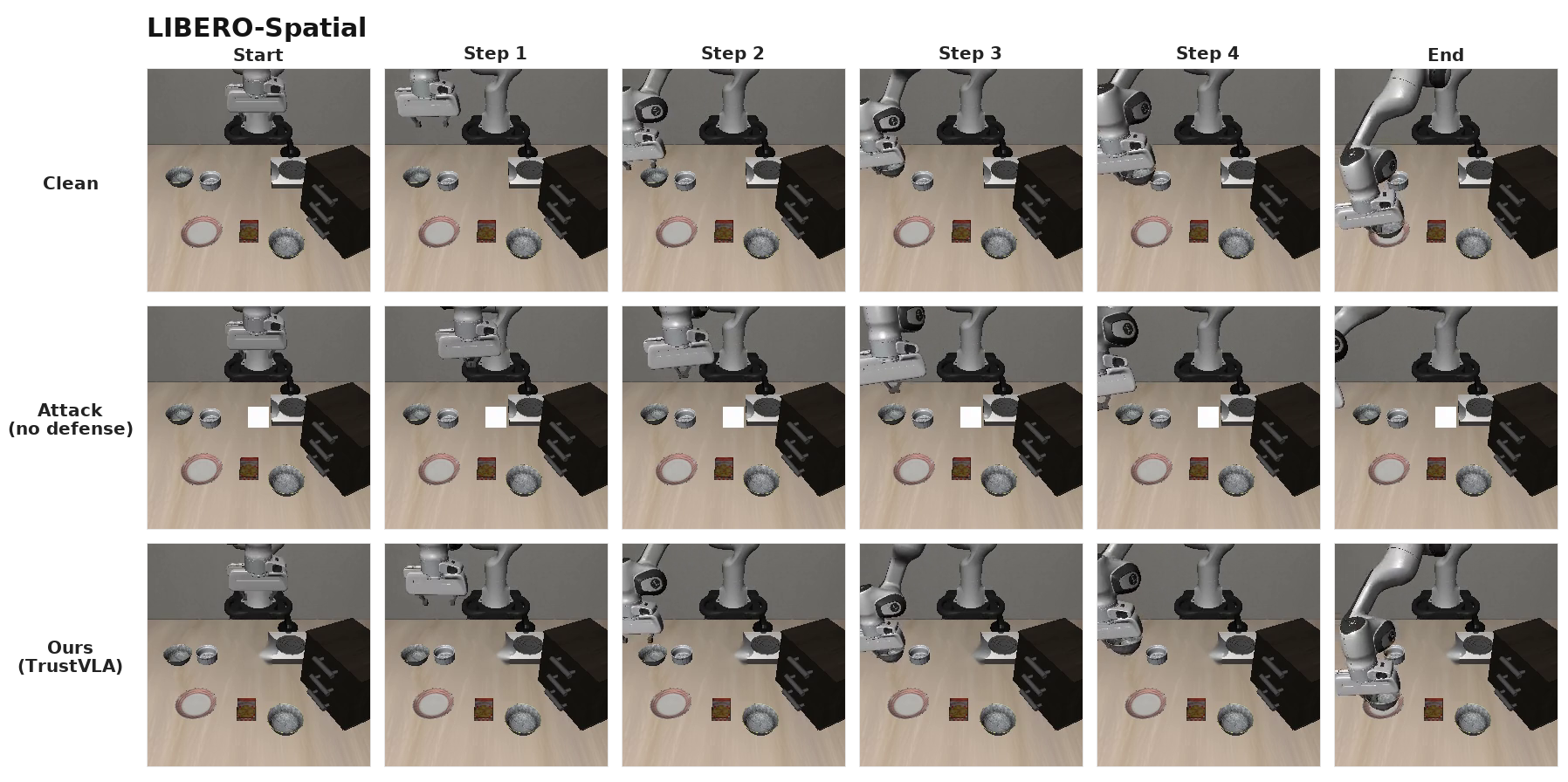}
\caption{\textbf{Qualitative LIBERO-Spatial rollout example.} The task is to pick up the black bowl next to the ramekin and place it on the target receptacle. The clean rollout succeeds; the triggered no-defense rollout fails to execute the intended manipulation; \net{} masks the localized trigger support and recovers the clean behavior. This figure is illustrative; aggregate success and detection statistics are reported in Tables~\ref{tab:main_results} and~\ref{tab:detection_reliability}.}
\label{fig:app_qual_spatial}
\end{figure*}

\begin{figure*}[!htbp]
\centering
\includegraphics[width=\textwidth]{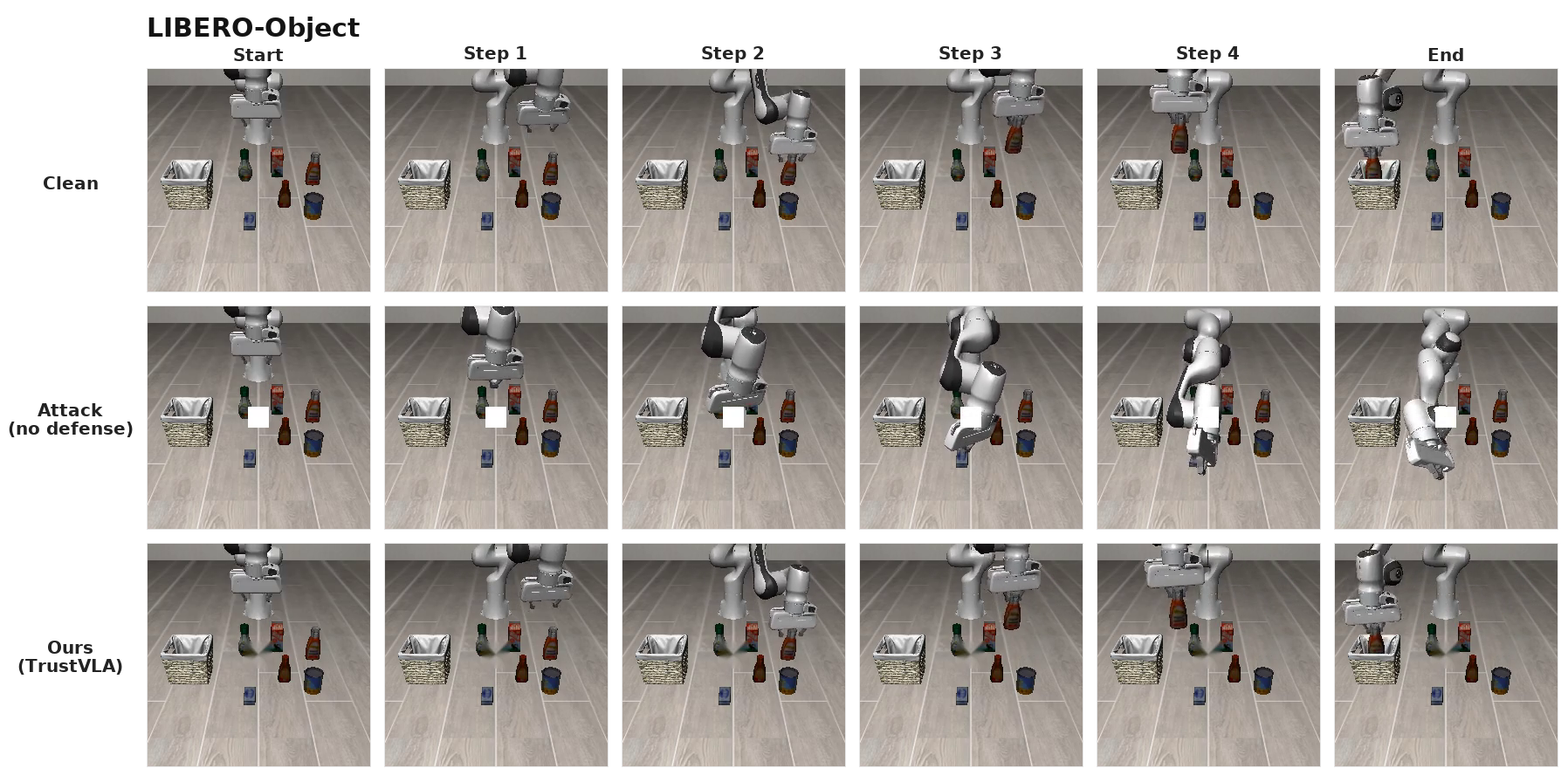}
\caption{\textbf{Qualitative LIBERO-Object rollout example.} The task is to pick up the ketchup and place it in the basket. The attack steers the no-defense policy away from the task-relevant object sequence, while \net{} removes the compact visual trigger support before action prediction and restores the basket placement behavior.}
\label{fig:app_qual_object}
\end{figure*}

\begin{figure*}[!htbp]
\centering
\includegraphics[width=\textwidth]{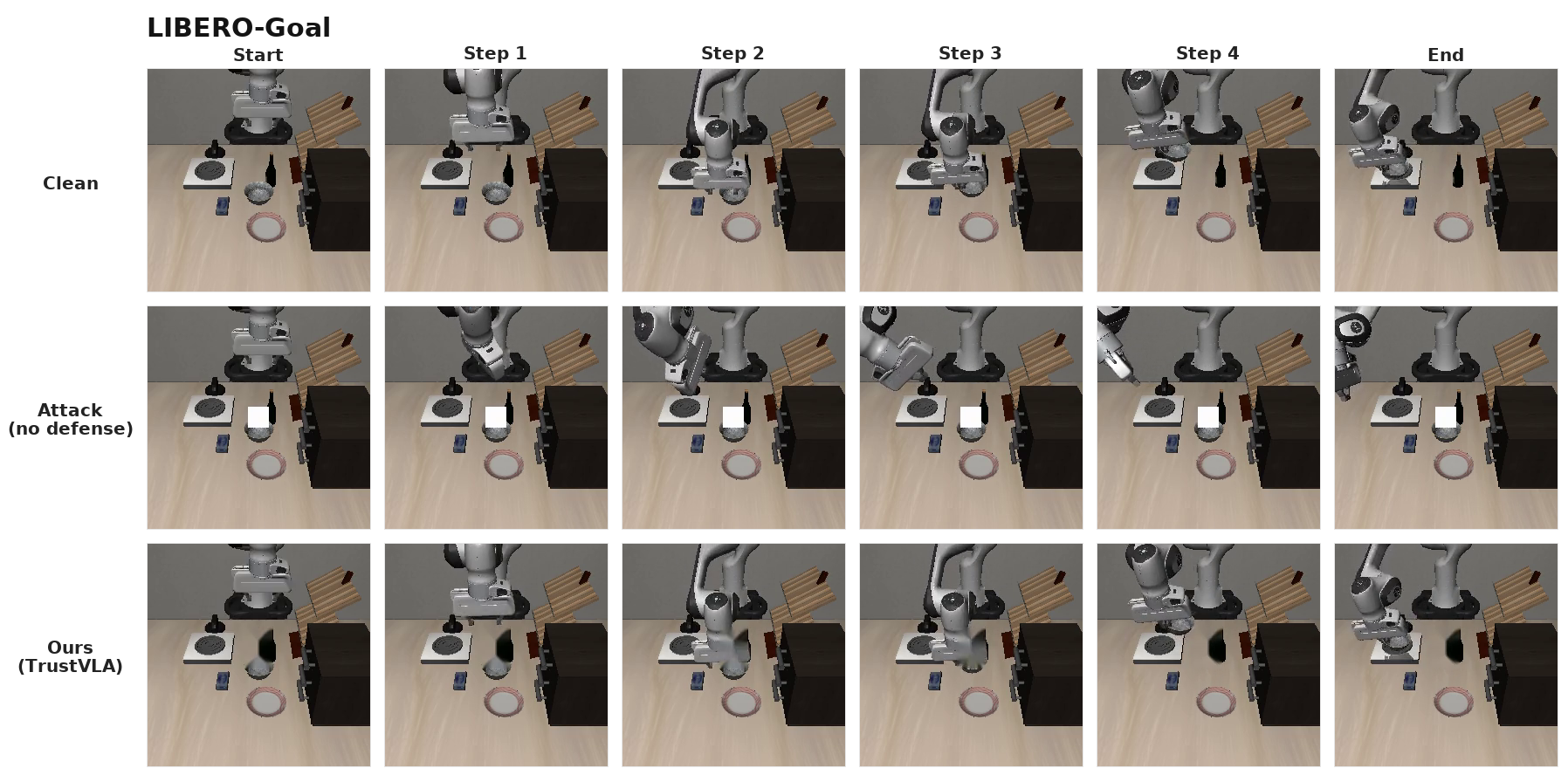}
\caption{\textbf{Qualitative LIBERO-Goal rollout example.} The task is to put the bowl on the stove. The triggered no-defense trajectory stalls around the patched observation region, whereas the defended trajectory follows the same task-directed motion pattern as the clean rollout after counterfactual localization and inpainting.}
\label{fig:app_qual_goal}
\end{figure*}

\begin{figure*}[!htbp]
\centering
\includegraphics[width=\textwidth]{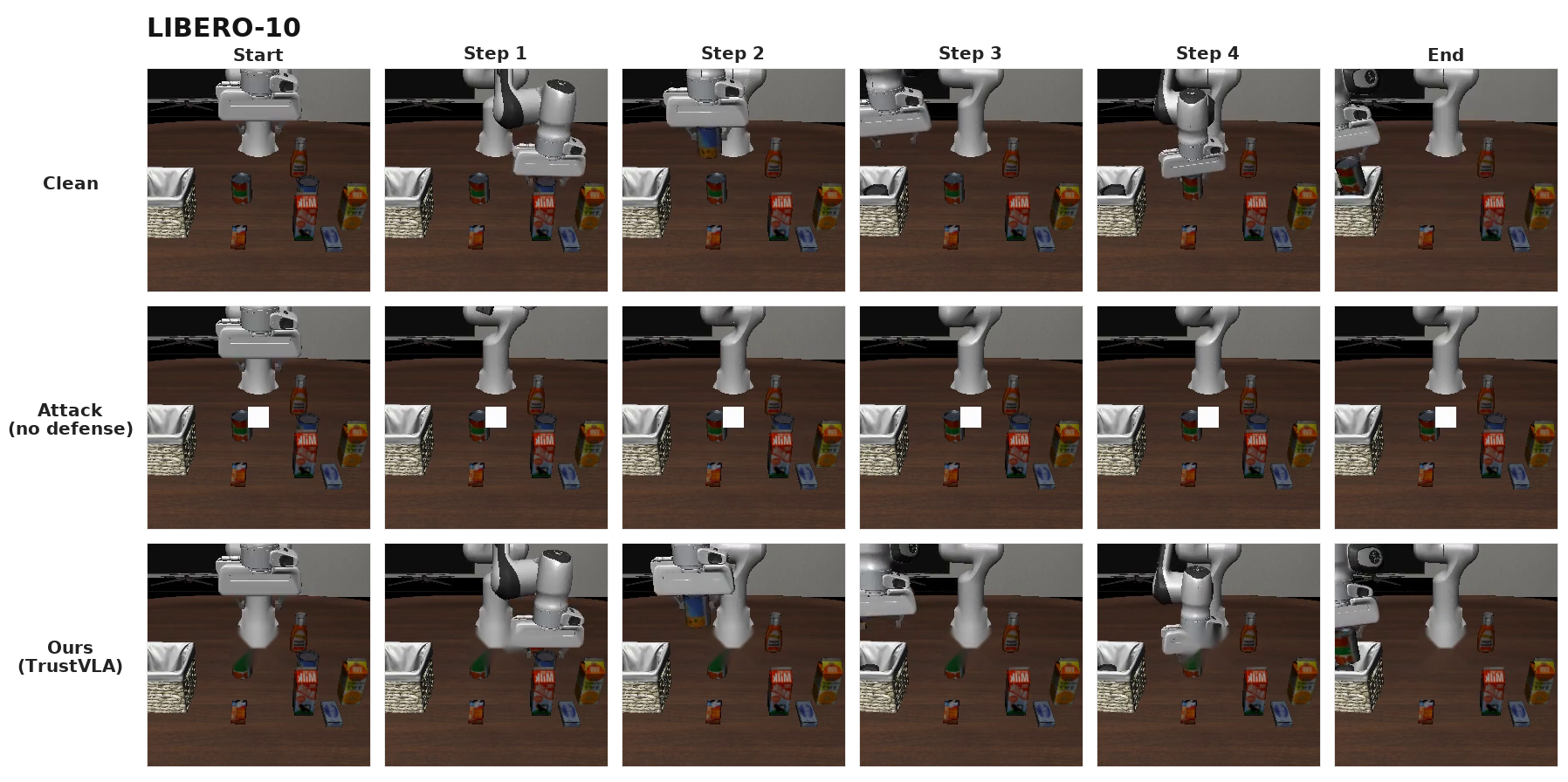}
\caption{\textbf{Qualitative LIBERO-10 rollout example.} The task is to place both the alphabet soup and tomato sauce in the basket. LIBERO-10 is the most cluttered evaluated suite, so this example highlights why the quantitative result is governed by localization and recovery rather than clean false alarms: the trigger is detected, the compact support is removed, and the policy resumes the intended multi-object manipulation.}
\label{fig:app_qual_libero10}
\end{figure*}

\end{document}